\documentclass[11pt, a4paper, logo, onecolumn,copyright]{tmlrmel}

\usepackage[authoryear, sort&compress, round]{natbib}
\usepackage{times}
\usepackage{latexsym}
\tcbuselibrary{breakable}
\tcbset{enhanced, attach boxed title to top left={yshift=-3mm,yshifttext=-1mm, xshift=3mm}, boxed title style={size=small}}

\usepackage[table]{xcolor}
\usepackage{graphicx}
\usepackage{mathtools}
\usepackage{amssymb}
\usepackage{bbm}
\usepackage{amsmath}
\usepackage{amsthm}
\usepackage{tabularx}
\usepackage{multirow, threeparttable, booktabs, makecell, adjustbox}
\usepackage{caption}
\usepackage{tcolorbox}
\usepackage{graphicx}
\usepackage{subcaption}
\newtheorem{theorem}{Theorem}[section]
\newtheorem{assumption}{Assumption}[section]
\newtheorem{definition}{Definition}[section]
\newtheorem{lemma}[theorem]{Lemma}
\newcommand\bs\boldsymbol
\tcbuselibrary{theorems}
\usepackage{enumitem}
\usepackage[T1]{fontenc}
\usepackage{microtype}
\usepackage{inconsolata}
\usepackage{graphicx}
\usepackage{enumitem}
\usepackage{wrapfig}

\newcounter{mybox}
\renewcommand{\themybox}{\thesection.\arabic{mybox}}
\newtcolorbox{mybox}[1][]{
  colback=gray!5,
  colframe=violet,
  fonttitle=\bfseries,
before upper={\refstepcounter{mybox}},
title={Box~\themybox\if\relax\detokenize{#1}\relax\else: #1\fi},
}

\title{LAD: Learning Advantage Distribution for  Reasoning}

\author{Wendi Li}
\author{Sharon Li}

\affil{Department of Computer Sciences\\University of Wisconsin--Madison}

\correspondingauthors{\{\texttt{wli679}, \texttt{sharonli}\}\texttt{@cs.wisc.edu}}

\coderepo{February 17, 2026}

\begin{abstract}
Current reinforcement learning objectives for large-model reasoning primarily focus on maximizing expected rewards. This paradigm can lead to overfitting to dominant reward signals, while neglecting alternative yet valid reasoning trajectories, thereby limiting diversity and exploration.
To address this issue, we introduce Learning Advantage Distributions (LAD), a distribution-matching framework that replaces advantage maximization with learning the advantage-induced distribution. By establishing the equivalence between the optimal policy update and an advantage-based target distribution, we derive a practical LAD objective formulated as minimizing an $f$-divergence between the policy-induced and advantage-induced distributions. This yields a gradient update that increases likelihood for high-advantage responses while suppressing over-confident probability growth, preventing collapse without requiring auxiliary entropy regularization.
LAD incurs no extra training cost compared to GRPO and scales naturally to LLM post-training. In a controlled bandit setting, LAD faithfully recovers the multimodal advantage distribution, validating the theoretical formulation. Experiments on math and code reasoning tasks across several LLM backbones show that LAD reliably improves both accuracy and generative diversity. Code is available \href{https://github.com/WindyLee0822/LAD}{here}.
\end{abstract}

\begin{document}
\maketitle

\vspace{-0.5em}
\section{Introduction}

Large language models have recently shown impressive performance across a wide range of complex reasoning tasks \citep{openai-o1,deepseek-r1,qwen3}. A central approach driving this progress is reinforcement learning with verifiable rewards (RLVR) \citep{deepseek-r1,grpo}. In RLVR, response correctness is automatically verified using external tools such as compilers, symbolic solvers, or mathematical checkers. By replacing noisy and unreliable feedback with deterministic, verification-based rewards, RLVR enables more stable optimization and mitigates reward hacking, leading to substantial gains in complex logical reasoning tasks~\citep{spurious,oneexample,drgrpo, deepcoder,code1,code2, agent3,agent1,agent2}.

However, current RLVR objectives for large-model reasoning are built around maximizing expected reward, a formulation that inherently overfits to dominant reward signals. This can drive policies toward mode collapses, suppressing alternative reasoning trajectories and reducing exploratory diversity \citep{entropy-mechanism,invisibleleash}. 
As illustrated in Figure~\ref{fig:toy-dist} (left), the policy distribution trained with the reward maximization objective (GRPO) concentrates on the most desirable action and neglects other meaningful modes. This reveals a fundamental limitation: the reward-maximization paradigm at the heart of RLVR imposes a structural bottleneck on the diversity and exploration~\citep{diverse-vital, flowrl}. Although entropy-based constraints and advantage regularization methods have been proposed to mitigate this issue~\citep{dapo, entropy-in-adv, 8020, decompose-entropy}, these approaches largely retain reward expectation maximization as their core objective. 

To address this limitation, we propose a framework that makes the policy \textbf{L}earn the \textbf{A}dvantage \textbf{D}istribution (\textbf{LAD}). Our work takes a perspective shift: we formalize policy optimization as learning to match an advantage-induced distribution, rather than maximizing expected advantage. This shift turns RLVR from a scalar-expectation objective into a distribution-matching problem, allowing the policy to preserve multiple promising reasoning modes instead of collapsing onto a single high-reward trajectory (Figure~\ref{fig:toy-dist}, right). We develop both the theoretical LAD objective and a practical surrogate that avoids intractable computation while preserving the same optima. The training cost of LAD is the same as standard GRPO, making LAD scalable to real-world LLM reasoning tasks. Concurrent work such as FlowRL~\citep{flowrl} also departs from pure reward maximization with a reward-matching objective, but it does not learn the full advantage-induced action distribution and can be proved as a more constrained instance of our broader distribution-matching framework.

Finally, we validate LAD through both controlled experiments and LLM reasoning evaluations. In a carefully designed bandit environment, LAD can learn the policy distribution that faithfully aligns with the underlying advantage-induced distribution. Additionally, we conduct experiments on six challenging mathematical reasoning and three coding reasoning benchmarks across multiple LLM backbones. LAD consistently improves both accuracy and diversity over GRPO, entropy-regularized baselines, and FlowRL. Further, we conduct comprehensive ablations to understand LAD's gains across divergence choices and training conditions.  Collectively, these findings validate LAD as a principled and effective training objective for RL-based post-training of LLMs, particularly in complex reasoning settings. We summarize our contributions as follows:
\begin{enumerate}[itemsep=0.0em, parsep=0.1pt, topsep=0pt, itemindent=0.0em]
    \item We introduce LAD, a new objective that reframes RLVR as matching the advantage-induced action distribution instead of maximizing expected advantage.
    \item We derive a practical surrogate with theoretical justification and improved optimization behavior, enabling scaling to real-world LLM reasoning tasks with comparable training cost as GRPO.
    \item Through extensive experiments and systematic ablations, we provide empirical evidence that LAD achieves superior performance and improves diversity.
    
\end{enumerate}

\begin{figure*}[t!]
    \centering
    \includegraphics[width=0.98\linewidth]{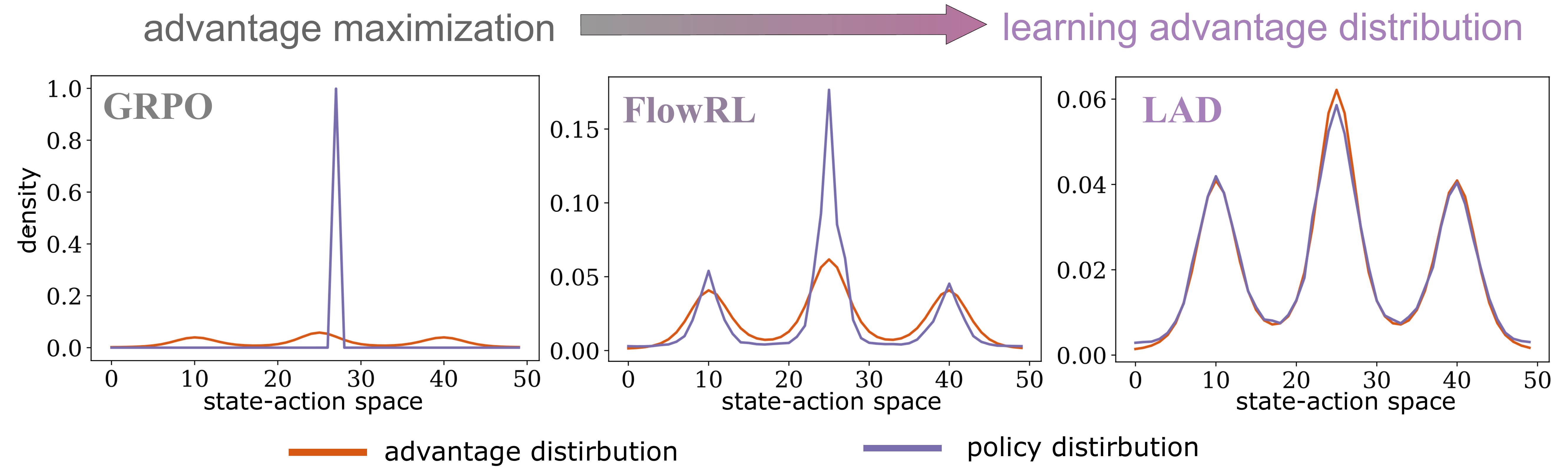}
    \vspace{-0.8em}
    \caption{Advantage distribution and policy distribution on 50-arm bandit trained via different training objectives.}
    \vspace{-1em}
    \label{fig:toy-dist}
\end{figure*}

\section{Preliminary}

\paragraph{Reward Maximization in RL for LLMs.} 
We consider a contextual bandit formulation of reinforcement learning (RL) for LLMs. Let $x \in \rho$ denote a prompt sampled from a distribution $\rho$, and let the model (policy) $\pi_\theta(y|x)$ generate a response $y \in \mathcal{Y}$. Upon observing the response, the environment returns a scalar reward $r(x,y) \in \mathbb{R}$. The objective of most existing RL algorithms is to learn policy parameters $\theta$ that maximize the expected reward, $J(\pi_\theta) = \mathbb{E}_{x \sim \mathcal{D},\, y \sim \pi_\theta(\cdot|x)}[r(x,y)]$.
Gradient-based methods \citep{reinforce} optimize this objective by exploiting the identity between $\nabla_\theta J(\pi_\theta)$ and
\begin{equation}
\mathbb{E}_{x\sim\rho,y\sim\pi_\theta(\cdot|x)}\!\left[-A(x,y)\nabla_\theta \log \pi_\theta(y|x) \right],  
\label{eq:reinforce}
\end{equation}
where $A(x,y)$ is an advantage function, commonly approximated by $r(x,y) - b(x)$ with $b(x)$ serving as a baseline to reduce gradient variance.

\paragraph{Trust-Region Constrained Policy Optimization.} 
Direct gradient ascent on $J(\pi_\theta)$ can induce large policy updates, often resulting in unstable training. Trust-region methods mitigate this issue by constraining each policy update to remain close to a behavior (sampling) policy $\pi_{\text{old}}$. A canonical formulation underlying Trust Region Policy Optimization~\citep{trpo} and its practical variant Proximal Policy Optimization~\citep{ppo} is given by:
\begin{align}
\max_{\theta} & \mathbb{E}_{x\sim\rho,y \sim \pi_{\text{old}} (\cdot|x)}  \!\left[\frac{\pi_\theta(y|x)}{\pi_{\text{old}}(y|x)} \cdot A(x,y)\right] \label{eq:trpo}\\
\qquad & \text{s.t.} \quad 
\mathbb{E}_{x \sim \rho }\!\left[D_{\mathrm{KL}}\!\left(\pi_{\text{old}}(\cdot|x)\,\|\,\pi_\theta(\cdot|x)\right)\right] \leq \epsilon,    \nonumber
\end{align}
where $\epsilon > 0$ controls the maximum allowable deviation from the behavior policy via a KL-divergence constraint.

\vspace{0.2cm}
\noindent \textbf{Reinforcement Learning with Verifiable Reward (RLVR)} is a post-training paradigm in which rewards are assigned by deterministic, rule-based verifiers. Because such rewards are less susceptible to reward hacking, RLVR has proven particularly effective for reasoning-intensive tasks. Representative RLVR algorithms, including GRPO \citep{grpo} and RLOO \citep{rloo}, typically incorporate group-wise normalization or ranking-based estimators to reduce variance in advantage estimation.
Despite these modifications, the core objective of existing RLVR methods remains the maximization of expected advantage, as in Eq.~\eqref{eq:reinforce}. This objective often leads to a rapid collapse of exploratory diversity, which tends to concentrate probability mass on a small set of high-reward responses and suppress alternative valid reasoning paths~\citep{entropy-mechanism,invisibleleash}. Motivated by this limitation, we propose to shift the learning objective from advantage maximization to learning the distribution of advantages.

\section{Method}

{In this section, we propose \textbf{Learning Advantage Distribution (LAD)}, a framework that trains policies by matching the distribution induced by advantages rather than maximizing expected advantage. Building on the structure implied by trust-region policy optimization, we formulate LAD as minimizing an $f$-divergence between a policy-induced distribution and an advantage-induced distribution (Sec.~\ref{sec:lad-objective}). We further derive a practical surrogate loss that preserves the same optimal policy while enabling scalable training for LLMs (Sec.~\ref{sec:practical-imp}).}

\subsection{Learn Advantage Distribution with $f$-divergence}
\label{sec:lad-objective}

To align the policy with the distribution induced by advantage values, we first define two distributions: one derived from the policy $\pi_\theta(y|x)$ and the other from the advantage function $A(x,y)$. At optimality, these two distributions coincide. We state this relationship through the following lemma, which directly motivates our training objective.
\begin{lemma}[\textbf{Distribution equivalence}]
Let $\pi_\theta$ denote the current policy parameterized by $\theta$, and let $\pi_{\mathrm{old}}$ be the behavior model. Suppose that $\pi_\theta$ is optimized using a trust-region constrained reinforcement learning algorithm (e.g., PPO \citep{ppo}) as in Eq.~\ref{eq:trpo}. There exists a Lagrange multiplier $\eta>0$, such that the optimal policy $\pi_\theta^*$ satisfies
    \begin{equation}
        \frac{\frac{\pi_\theta^*(y|x)}{\pi_\mathrm{old}(y|x)}}{Z_\pi(x)} = \frac{e^{\frac{A(x,y)}{\eta}}}{Z_A(x)} ,
    \end{equation}
where $\eta$ is a Lagrangian multiplier, $Z_\pi(x) = \sum_y \frac{\pi_\theta^*(y|x)}{\pi_\mathrm{old}(y|x)}$ and $Z_A(x)=\sum_y e^{\frac{A(x,y)}{\eta}}$ are two normalizations.
\label{lemma:two-distribution}
\end{lemma}
Lemma~\ref{lemma:two-distribution} implies that trust-region constrained optimization implicitly induces a policy-improvement probability distribution proportional to the exponentiated advantage. 
{Rather than treating this form as a theoretical result, we make it explicit and treat it as a learning target. This leads to a shift in perspective: policy optimization can be viewed as learning to match an advantage-induced distribution, rather than maximizing expected advantage.}

Formally, we define two distributions below:
\begin{align}
    \mathcal{P_{\pi_\theta}}(y|x) &= \frac{\pi_\theta (y|x)}{\pi_\mathrm{old}(y|x)Z_{\pi}(x)} ,\label{eq:pi}\\
     \mathcal{P_A}(y|x) &= \frac{e^\frac{A(x,y)}\eta}{Z_{A}(x)}.\label{eq:a}
\end{align}
Intuitively, $\mathcal{P_A}$ is the target advantage distribution, encouraging higher probability to higher-advantage responses; $\mathcal{P_{\pi_\theta}}$ is the policy-induced distribution measuring how much the policy deviates from the behavior policy. The constructions of $\mathcal{P_A}$ and $\mathcal{P_{\pi_\theta}}$ are directly inspired by the optimal policy form arising from trust-region constrained policy optimization. This connection ensures that the proposed distributions are theoretically consistent with established policy improvement principles.  

Given these two equivalent distributions, we formulate the learning objective by minimizing their divergence. To this end, we recall the definition of the $f$-divergence, which is a general class of discrepancy measures that subsumes many commonly used divergences (e.g., KL, Jensen-Shannon, Hellinger etc.).
\begin{definition}[$f$-divergence]
Let $\mathcal{P}$ and $\mathcal{Q}$ are two distributions, then for a convex function $f:[0,+\infty)\rightarrow (-\infty,+\infty)$ such that $f(x)$ is finite for all $x>0$, and $f(1)=0$, the $f$-divergence of $\mathcal{P}$ from $\mathcal{Q}$ is defined as
\begin{equation}
    D_f(\mathcal{P}\Vert\mathcal{Q}) = \int_{\Omega} f(\frac{\mathrm{d} \mathcal{P}}{\mathrm{d} \mathcal{Q}}) \mathrm{d} \mathcal{Q}.
\end{equation}
\label{def:f-div}
\end{definition}
\paragraph{LAD: Minimizing the divergence between $\mathcal{P_{\pi_{\theta}}}$ and $\mathcal{P_A}$.} Using this measure of distributional discrepancy, we propose to train the policy by minimizing the $f$-divergence between $\mathcal{P_{\pi_{\theta}}}$ and $\mathcal{P_A}$, thereby encouraging the policy to learn the advantage distribution. This divergence-minimization objective constitutes the central novelty of LAD:
\begin{align}
\mathcal{L}_\mathrm{LAD}^\mathrm{theorem}& = \mathbb{E}_{x\sim\rho} \sum_y \frac{e^{\frac{A(x,y)}{\eta}}}{Z_A(x)} f(\frac{\frac{\pi_\theta(y|x)}{\pi_\mathrm{old}(y|x)}}{e^\frac{A(x,y)}{\eta}}\cdot\frac{Z_A(x)}{Z_{\pi_\theta}(x)}).
     \label{eq:theoretical}
\end{align}

Importantly, minimizing the LAD objective preserves the same optimal policy as standard trust-region policy optimization, namely $\pi^* \propto \pi_\mathrm{old} \cdot e^{\frac{A(x,y)}{\eta}}$.
While the resulting theoretical formulation is principled, it involves normalization terms $Z_A(x)$ and $Z_{\pi_\theta}(x)$ that are intractable for large action spaces such as those encountered in large language models.
To address this challenge, we introduce a practical surrogate objective in the next section that is computationally tractable while preserving the same optimal policy.

\subsection{Practical Implementations}
\label{sec:practical-imp}
Instead of approximating the intractable normalization constants directly, our key insight is  that the optimal policy of an $f$-divergence objective depends only on the relative structure of its weighting functions. This  allows us to replace the theoretical LAD objective with a surrogate loss that is tractable yet induces the same optimal policy.
To facilitate this transition, we introduce the following lemma.
\begin{lemma}[{Linking loss formulations to their induced optimal policies.}]
   Let $\mathcal{L} = \mathbb{E}_{x\sim\rho} \sum_y g_2(x,y) f(\frac{p_\theta(y|x)}{ g_1(x,y)})$, where $p_\theta(\cdot)$ is a probability measure with $\sum_y p(\pi)=1$. $g_1(x,y) , g_2(x,y)$ are functions independent of $p_\theta$. If the ratio $\frac{g_2(x,y)}{g_1(x,y)} = Z_g(x)$ depends only on $x$ and is independent of $y$, then any optimal solution \(p_\theta^\ast\) of \(\mathcal{L}\) satisfies $p^*_\theta(y|x) \propto g_1(x,y)$.
   \label{lemma:equivalence}
\end{lemma}
\paragraph{Interpretation.} Applying Lemma~\ref{lemma:equivalence} to the theoretical LAD objective immediately yields the optimal policy $\pi_\theta^*\propto \pi_\mathrm{old} \cdot e^\frac{A(x,y)}\eta$.
Crucially, the lemma also implies that this optimal policy can be recovered by alternative loss formulations, provided that the induced ratio between the weighting functions remains invariant with respect to $y$. In particular, when setting $p_\theta = \pi_\theta$ and choosing $g_1(x,y)= g_2(x,y) = \pi_\mathrm{old} \cdot e^\frac{A(x,y)}\eta$, the resulting objective admits the same optimal policy, $\pi_\theta^*\propto \pi_\mathrm{old} \cdot e^\frac{A(x,y)}\eta$, while avoiding the intractable normalization terms present in the theoretical LAD loss.  This observation enables transforming the theoretical LAD objective into a practical surrogate that is computationally tractable while preserving the same optimal policy. 

\begin{tcolorbox}[width=\linewidth, colback=white!95!black, left=1mm, right=1mm, top=0.5mm, bottom=0.5mm]
\textbf{LAD practical objective:}
\vspace{-0.3cm}
\begin{equation}
    \mathcal{L}_\mathrm{LAD} = \mathbb{E}_{x\sim\rho,y\sim\pi_\mathrm{old}(\cdot|x)}  {e^\frac{A(x,y)}{\eta}} f(\frac{\frac{\pi_\theta(y|x)}{\pi_\mathrm{old}(y|x)}}{e^\frac{A(x,y)}\eta}).
    \label{eq:practical}
\end{equation}
This practical objective admits the same optimal policy as the theoretical loss $\mathcal{L}_\mathrm{LAD}^\mathrm{theorem}$, while exhibiting favorable optimization behavior and closely approximating it in practice.
\end{tcolorbox}

\subsection{Analyses and Discussions}
\label{sec:analyses-discussion}
\paragraph{$\mathcal{L}_\mathrm{LAD}$ yields desirable loss behavior with implicit regularization.}
The gradient of the practical LAD loss with respect to the policy parameters $\nabla_\theta \mathcal{L}_\mathrm{LAD}$ can be written as, 
\begin{align}
\mathbb{E}_{x\sim\rho,y\sim\pi_\mathrm{\theta}(\cdot|x)} f'(\frac{\frac{\pi_\theta(y|x)}{\pi_\mathrm{old}(y|x)}}{e^\frac{A(x,y)}\eta}) \cdot \nabla_\theta \log \pi_\theta (y|x). 
\end{align}
In contrast to the vanilla policy gradient in Eq.~\eqref{eq:reinforce}, which increases the log-likelihood of all actions with positive advantage, the weighting term in the LAD gradient, $f'(\frac{\pi_\theta(y|x)}{\pi_\mathrm{old}(y|x)}\cdot e^\frac{-A(x,y)}\eta)$ depends jointly on the advantage $A(x,y)$ and the likelihood ratio $\frac{\pi_\theta(y|x)}{\pi_\mathrm{old}(y|x)}$. As a result, even when a trajectory $(x,y)$ has a large positive advantage, the LAD gradient naturally suppresses further increases in $\log \pi_\theta(y|x)$ once the likelihood ratio becomes sufficiently large. This adaptive weighting prevents uncontrolled amplification of action probabilities and induces an \emph{implicit regularization} effect, leading to more stable optimization and improved diversity preservation.
Unlike PPO-style clipping as detailed in Appendix~\ref{sec:apd:gradient-comparison}, which enforces hard constraints on likelihood ratios, LAD induces a soft, advantage-aware suppression mechanism that emerges naturally from distribution matching.

\paragraph{$\mathcal{L}_\mathrm{LAD}$ is a close surrogate to the theoretical LAD objective $\mathcal{L}_\mathrm{LAD}^\mathrm{theorem}$.}
Beyond inducing desirable gradient behavior, the practical LAD objective also closely approximates the theoretical LAD formulation, thereby serving as a principled and reliable surrogate. A formal bound is given in Theorem~\ref{thm:surrogate-formal}, and additional discussion and empirical validation are provided in Appendix~\ref{sec:apd:theorem}.
Finally, we summarize several representative choices of divergence functions within the LAD framework in Table~\ref{tab:lad-losses}.

\begin{figure}[t!]
    \centering
    \includegraphics[width=0.49\linewidth]{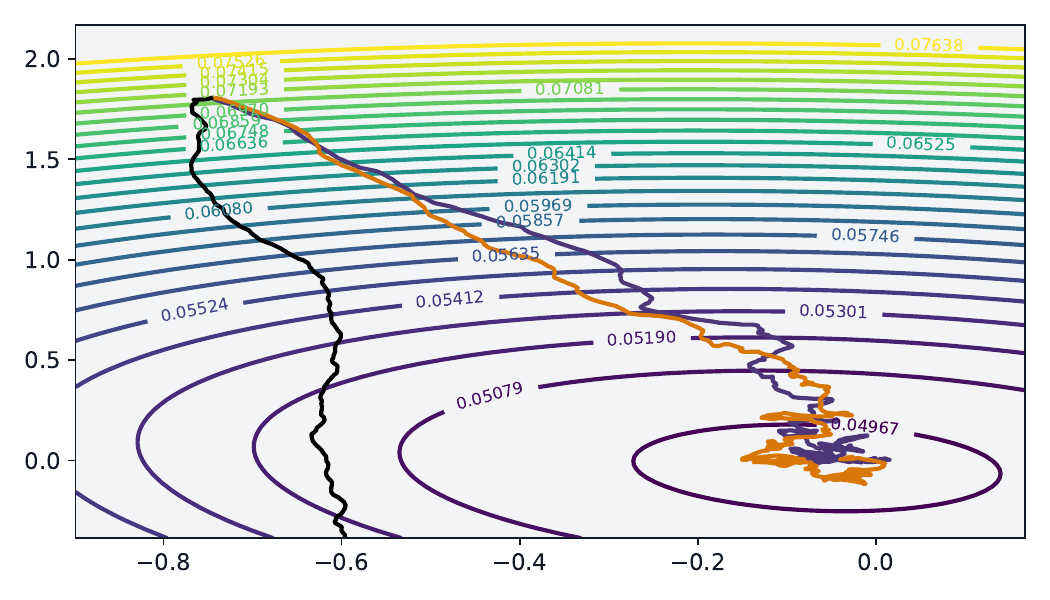}
    \vspace{-0.5em}
    \includegraphics[width=0.49\linewidth]{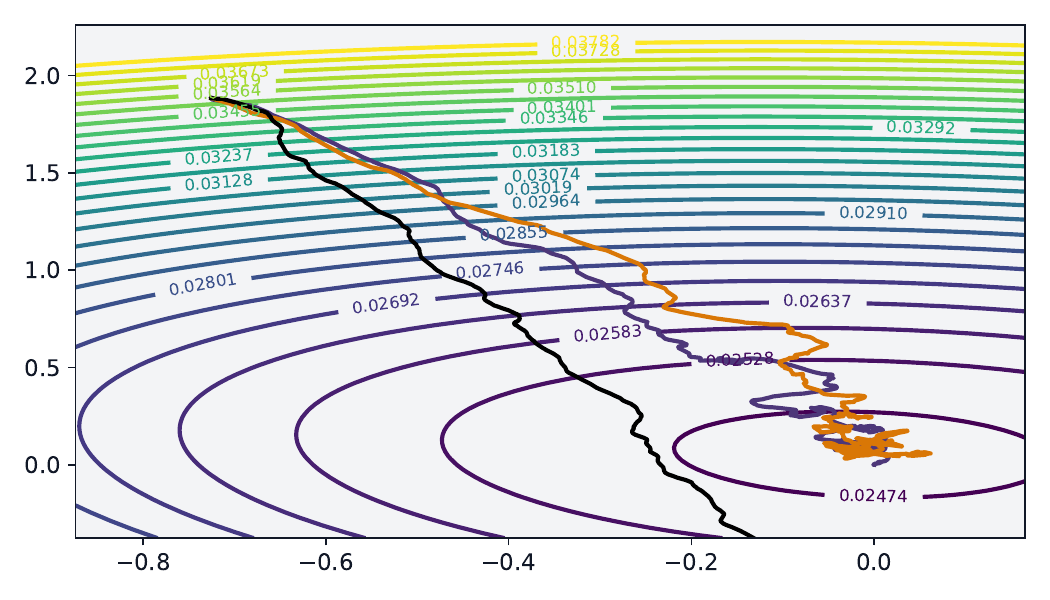}
    \vspace{-1em}
    \caption{Learning dynamics of GRPO, \textcolor{orange}{the theoretical LAD loss}, and \textcolor{violet}{the practical LAD loss}. The loss landscape is constructed with the practical LAD loss. Top:  LAD loss under Hellinger distance. Bottom: LAD loss under Jensen–Shannon divergence.}
    \label{fig:losslandscape}
\end{figure}

\section{Experiments}

We empirically evaluate LAD to answer three key questions: (i) Does the LAD objective successfully learn advantage distributions in controlled settings? (ii) Does this translate into improvements on large-scale reasoning tasks? (iii) How do different divergence choices and hyperparameters influence performance? To this end, we begin with a controlled distribution-matching task, followed by extensive LLM experiments, and conclude with a set of ablation studies.
\subsection{Controlled Experiments}

We begin with a controlled bandit environment to visualize the behavior of different objectives. This setting \emph{allows us to directly compare the learned policy distribution against the ground-truth advantage distribution}, revealing whether an objective genuinely performs distribution matching.

\paragraph{Settings.}
We design a 50-bandit problem. The policy model is parameterized by a $1\times50$ vector, where each entry represents the probability of selecting a corresponding arm. The policy model $\pi_\theta$ and the behavior model $\pi_\mathrm{old}$ are initialized to the uniform distribution. The underlying advantage distribution is constructed as a mixture of three Gaussian distributions, resulting in three distinct modes. Training is performed for 4000 steps. At each step, the behavior policy $\pi_\mathrm{old}$ is sampled 32 times with a sampling temperature of 1.5. The behavior policy is updated twice throughout training. We use a learning rate of 5e-3. For LAD loss variants, $\eta$ is set to $1$. For GRPO, we additionally apply a KL-divergence regularization between $\pi$ and $\pi_\mathrm{old}$ with weight $1$. Under this setting, the theoretical optimal policy of GRPO coincides with that of the LAD objective, enabling a fair comparison.

\paragraph{LAD loss successfully matches the two distributions.}
Figure~\ref{fig:toy-dist} visualizes the final learned policy distributions alongside the ground-truth advantage distribution. The advantage distribution is shown in orange, while the learned policy distribution is shown in blue. The LAD objective $\mathcal{L}_\mathrm{LAD}$ enables the policy-induced distribution $\mathcal{P}_{\pi_\theta}$ to closely match the target advantage distribution $\mathcal{P}_{A}$, faithfully recovering all three modes. 
In contrast, the advantage-maximization objective of GRPO concentrates probability mass on a single arm. This leads to a collapse of the policy distribution and a substantial reduction in generative diversity during training, ultimately resulting in suboptimal performance on the less preferred arm.

\renewcommand{\arraystretch}{0.9}
\setlength{\tabcolsep}{9pt}
\begin{table*}[t!]
    \centering
      \caption{Performance of different losses on math reasoning benchmarks. The best two results under Avg@32 are highlighted in \textbf{bold} and \underline{underlined}.}
    \label{tab:main-math}
    \vspace{-0.5em}
    \scalebox{0.9}{
    \begin{tabular}{c|ccccccc}
    \toprule
        Methods  & MATH & \makecell[c]{Olympiad\\Bench}& \makecell[c]{Minerva}& \makecell[c]{AIME\\2024}&\makecell[c]{AIME\\2025}& AMC& Average  \\
        \midrule
        {GRPO}{\footnotesize \citep{grpo}} &  75.57&39.47&36.32&14.03&4.86&53.44& 37.28\\
        EntAdv{\footnotesize \citep{entropy-in-adv}}& 76.60&\underline{41.16}&35.68&15.00&7.23&52.97&{38.11} \\
        KLCov{\footnotesize \citep{entropy-mechanism}}& 76.40&{40.13}& 36.28& 14.66& \textbf{9.81}& \underline{53.45}& \underline{38.46} \\
        ClipCov{\footnotesize \citep{entropy-mechanism}}& 76.03& 40.30&\textbf{37.58} & \underline{16.29} & 7.39 & 52.27&38.31 \\
       {FlowRL}{\footnotesize \citep{flowrl}} &\underline{77.17}&{40.80}&36.21&14.67&7.51&47.34& 37.28\\
       \midrule
        
        {LAD} &\textbf{77.66}&\textbf{41.25}&\underline{36.68}&\textbf{19.60}&\underline{8.84}&\textbf{56.45}& \textbf{40.08}\\
       
        \bottomrule
    \end{tabular}}
\end{table*}
\begin{table*}[]
    \centering
    \caption{Performance of different losses on code reasoning benchmarks. The best two results under a specific metric are highlighted in \textbf{bold} and \underline{underlined}.}
    \label{tab:code-reasoning}
    \vspace{-0.5em}
    \scalebox{0.85}{
    \begin{tabular}{c|cc|cc|cc}
    \toprule
       \multirow{2}{*}{Method}  & \multicolumn{2}{c|}{LiveCodeBench} & \multicolumn{2}{|c|}{HumanEval+} & \multicolumn{2}{|c}{CodeForces}  \\
         &  Avg@16 & Pass@16 & Avg@16 & Pass@16 & Rating & Percentile \\
         \midrule
        GRPO{\footnotesize \citep{grpo}} & 32.46& 51.25& 80.90 & 93.25& 1355.35& 70.30 \% \\
        EntAdv{\footnotesize \citep{entropy-in-adv}}& 32.75 & 50.90& 81.05& 95.09 & 1382.78 & 72.20\% \\
        KLCov{\footnotesize \citep{entropy-mechanism}} & 29.50  &50.18 &80.75& 95.09&1373.07 & 71.60\% \\
        ClipCov {\footnotesize \citep{entropy-mechanism}}& 32.10& \textbf{51.97}& \underline{81.48}& 94.49& \underline{1473.09}& \underline{79.20\%} \\
        FLowRL {\footnotesize \citep{flowrl}}&\underline{33.24} & 51.61 & 81.29& 95.09& 1364.70 & 71.00\% \\
        \midrule
        LAD & \textbf{33.51} &\textbf{51.97}& \textbf{82.29} & \textbf{95.71} &  \textbf{1533.64}& \textbf{82.50\%} \\
        \bottomrule
    \end{tabular}}
    \vspace{-1em}
\end{table*}

\paragraph{LAD loss $\mathcal{L}_\mathrm{LAD}$ is an accurate surrogate of $\mathcal{L}_\mathrm{LAD}^\mathrm{theorem}$.}
We further examine the learning dynamics induced by $\mathcal{L}_\mathrm{LAD}$, $\mathcal{L}_\mathrm{LAD}^\mathrm{theorem}$ and GRPO.
The corresponding trajectories are visualized in Figure~\ref{fig:losslandscape}. To construct the loss landscape, we use the practical LAD objective and project the model parameters from each training step onto this landscape.
As shown in Figure~\ref{fig:losslandscape}, the optimization trajectories induced by the practical loss and the theoretical loss closely align, exhibiting highly similar learning dynamics. In contrast, GRPO follows a markedly different trajectory, deviating from the geometry implied by the theoretical LAD objective. These results empirically confirm that the practical LAD loss serves as a faithful and accurate surrogate for the theoretical formulation.

\paragraph{Comparison with FlowRL.} 
Concurrent with our work, FlowRL \citep{flowrl} proposes to move beyond reward maximization by adopting a reward-matching objective derived from the trajectory balance loss in GFlowNets \citep{gflownet}. While this formulation departs from conventional reward-maximization objectives, Figure~\ref{fig:toy-dist} (middle) shows that FlowRL fails to fully align the advantage-induced distribution $\mathcal{P_A}$ and the policy-induced distribution $\mathcal{P_{\pi_\theta}}$. 
As detailed in Appendix~\ref{sec:specialcase}, {FlowRL can in fact be cast as a special case of the LAD framework, but only under substantially stricter conditions} on policy updates and parameter scaling.
In contrast, LAD provides a more general and theoretically grounded framework that directly learns the advantage distribution through principled $f$-divergence minimization, thus allowing LAD to more faithfully capture the full structure of the advantage distribution.

\begin{table*}[t!]
    \centering
    \caption{Performance of LAD and FlowRL on different LLM backbones. The best result is highlighted in \textbf{bold}.}
    \vspace{-0.5em}
    \label{tab:backbone}
    \scalebox{0.88}{
    \begin{tabular}{c|ccccccc}
    \toprule
      Methods & MATH & \makecell[c]{Olympiad\\Bench}& \makecell[c]{Minerva}& \makecell[c]{AIME\\2024}&\makecell[c]{AIME\\2025}& AMC & Average \\
      \midrule
      \rowcolor{gray!30!white}  \multicolumn{8}{c}{\textcolor{violet}{\texttt{OpenMath-Nemotron-1.5B: Avg@32}}} \\
      FlowRL & 80.43& 45.81& 21.02& 21.36& 22.66& 55.47 & 41.13  \\

      LAD & \textbf{86.81}& \textbf{55.63}& \textbf{22.98}& \textbf{32.70}& \textbf{26.47}& \textbf{68.71} & \textbf{48.88}  \\
    \rowcolor{gray!30!white} \multicolumn{8}{c}{\textcolor{violet}{\texttt{OpenMath-Nemotron-1.5B: Pass@32}}}\\
      FlowRL& 94.60& 65.93& 38.24& 56.35& 37.60& 80.65 & 66.98 \\
      LAD&  \textbf{97.20} & \textbf{78.52}& \textbf{42.28}& \textbf{72.29}& \textbf{54.38}& \textbf{89.08 }& \textbf{72.29}\\ 
     \midrule
    \rowcolor{gray!30!white}  \multicolumn{8}{c}{\textcolor{violet}{\texttt{Thinkless-DeepScaleR-1.5B: Avg@32}}} \\
      FlowRL & 80.90& 42.71& \textbf{34.11}& 19.41& 16.90& 54.80& 41.47\\
      LAD & \textbf{83.01}& \textbf{46.08}& 33.71& \textbf{21.85}& \textbf{20.13}& \textbf{59.38}& \textbf{44.03 } \\

      \rowcolor{gray!30!white}  \multicolumn{8}{c}{\textcolor{violet}{\texttt{Thinkless-DeepScaleR-1.5B: Pass@32}}} \\
      FlowRL & 96.40& 69.19& 56.62& 51.56& \textbf{43.64}& 84.79& 67.03   \\
      LAD & \textbf{95.60}& \textbf{69.93}& \textbf{60.66}& \textbf{55.52}& 42.60& \textbf{86.62}& \textbf{68.49} \\
        
        \bottomrule
    \end{tabular}}
    \vspace{-1em}
\end{table*}

\subsection{LLM Experiments}

Having verified that LAD recovers the target distribution in a toy setting, we now test whether these benefits translate to real-world LLM tasks.

\paragraph{Experimental settings.}
We experiment on both the math and code domains. For the math domain, we use Qwen2.5-7B~\citep{qwen2.5} as our LLM backbone and DAPO-MATH~\citep{dapo} as our training dataset. The experimental protocol follows \citet{entropy-mechanism} with a dynamic sampling strategy and the clip-higher technique \citep{dapo}. Evaluation is conducted on a diverse set of mathematical reasoning benchmarks, including MATH500 \citep{prm800k}, AIME 2024, AIME 2025 \citep{aime}, AMC \citep{aime}, OlympiadBench \citep{olympiadbench}, and Minerva \citep{minerva}. For the code domain, we use DeepSeek-R1-Distill-7B as our LLM backbone and DeepCoder \citep{deepcoder} as our training dataset. We follow the experimental protocol of \citep{flowrl}, and evaluate the performance on LiveCodeBench \citep{livecodebench}, CodeForces \citep{codeforces}, and HumanEval+ \citep{humaneval}. For both domains, we employ GRPO as the advantage estimator and use Jensen-Shannon divergence to implement LAD loss. The learning rate is set to 1e-6. More implementation details are in Appendix~\ref{sec:apd:implement}.

\begin{table}[t]
    \caption{Generative diversity results of on rollouts of AIME 24/25 problems. The best two results under a specific metric are highlighted in \textbf{bold} and \underline{underlined}.}
    \label{tab:diversity}
    \centering
    \scalebox{0.82}{
    \begin{tabular}{c|cc|c}
    \toprule
        Methods & dist-3($\uparrow$) & dist-4($\uparrow$) & GPT4-Judge($\uparrow$)  \\
        \midrule
        GRPO &0.2306& 0.2902 &2.04\\
        EntAdv& 0.2978& 0.3782& 2.07\\
        KLCov& 0.2471&0.3262& 2.07\\
        ClipCov&\underline{0.3315}& \underline{0.4171}&2.10\\
        FLowRL & 0.2878&0.3654 & \underline{2.38}\\
        LAD &\textbf{0.3498}& \textbf{0.4442}&\textbf{2.58} \\
        \bottomrule
    \end{tabular}}
    \vspace{-1em}
\end{table}

\paragraph{Baselines.} 
We compare against the latest methods designed to mitigate entropy collapse and improve generative diversity, including EntAdv, KLCov, and ClipCov \citep{entropy-mechanism}, as well as FlowRL \citep{flowrl}. EntAdv, KLCov, and ClipCov are regulation-based approaches that introduce additional regularization terms into the loss function; however, they retain expected advantage maximization as the primary optimization objective. In contrast, FlowRL represents the most recent departure from advantage maximization by explicitly learning a reward-matching objective.

\paragraph{LAD leads to superior reasoning performance.}
Table~\ref{tab:main-math} reports the performance of different training objectives on a suite of mathematical reasoning benchmarks. Across nearly all benchmarks, LAD consistently outperforms prior methods, achieving the best average performance under both evaluation protocols. In particular, LAD attains the highest Avg@32 score overall (40.08), with substantial gains on challenging benchmarks such as AIME 2024 (+3.31 over the strongest baseline) and AMC (+3). Performance under the Pass@32 metric is also reported in Appendix~\ref{sec:apd:exp-results}, where LAD achieves a new state-of-the-art average score of 65.76, outperforming FlowRL by 2.63 points and regulation-based methods by an even larger margin. A similar performance advantage is observed on code reasoning tasks (Table~\ref{tab:code-reasoning}), where LAD achieves the best results on LiveCodeBench and HumanEval+ and the highest CodeForces percentile (82.50\%), indicating not only improved reasoning accuracy but also superior generalization to real-world programming problem difficulty. Taken together, these results demonstrate that LAD is broadly effective at improving both mathematical and code reasoning across diverse benchmarks and evaluation settings.

\vspace{-0.5em}
\paragraph{LAD boosts generative diversity.}
To evaluate the effectiveness of LAD in enhancing response diversity, we follow \citet{flowrl} and assess logical diversity using an LLM-as-a-Judge framework with GPT-4-Turbo on AIME 2024/2025 as detailed in Appendix~\ref{sec:apd:implement}. In addition to logical diversity, we measure lexical diversity using standard automatic metrics, distinct-$n$ with $n$=3,4. 
As shown in Table~\ref{tab:diversity}, LAD attains the highest distinct-3 and distinct-4 scores, indicating substantially richer lexical variation than all baselines. Meanwhile, LAD achieves the highest GPT-4 Judge score among all methods, suggesting that a diverse logical pattern accompanies its increased diversity.

\paragraph{LAD can be extended to different LLM backbones.}
Besides the Qwen2.5-7B for math reasoning and DeepSeek-R1-Distill-7B for code reasoning, we further include two other LLM backbones, OpenMath-Nemotron-1.5B and Thinkless-DeepScaleR-1.5B. As shown in Table~\ref{tab:backbone}, LAD still outperforms FlowRL on both backbones.

\subsection{Further Studies}
\label{sec:exp-further-study}

\begin{figure}
    \centering
    \begin{minipage}{0.54\textwidth}
        \centering
        \includegraphics[width=0.48\linewidth]{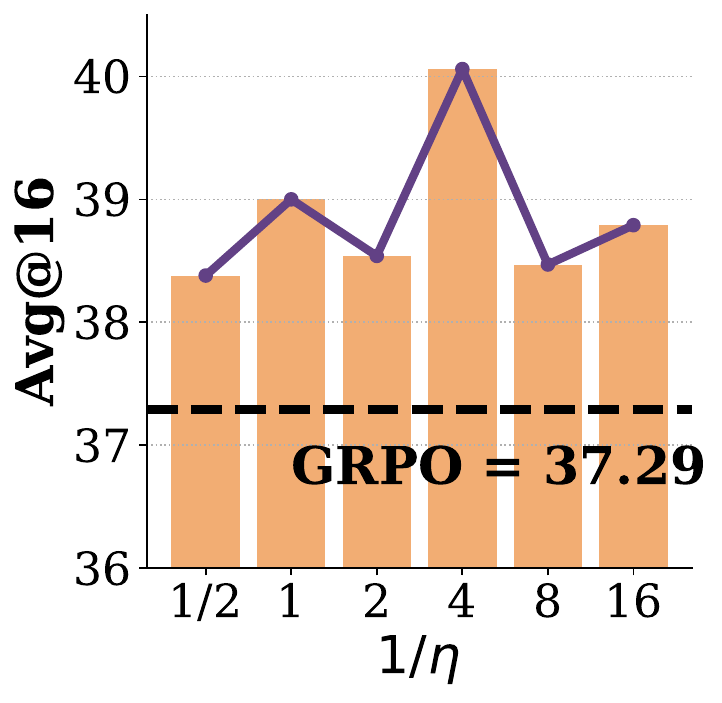}
        \includegraphics[width=0.48\linewidth]{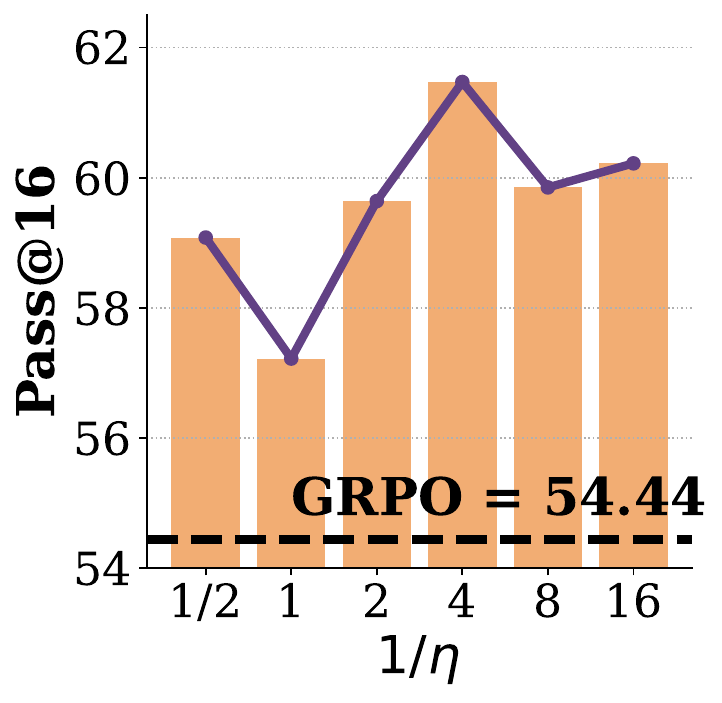}
        \caption{Average Avg@16 and Pass@16 scores over six mathematical benchmarks with different values of the hyperparameter $\eta$. }
        \label{fig:eta}
    \end{minipage}
    \hfill
    \begin{minipage}{0.45\textwidth}
        \includegraphics[width=\linewidth]{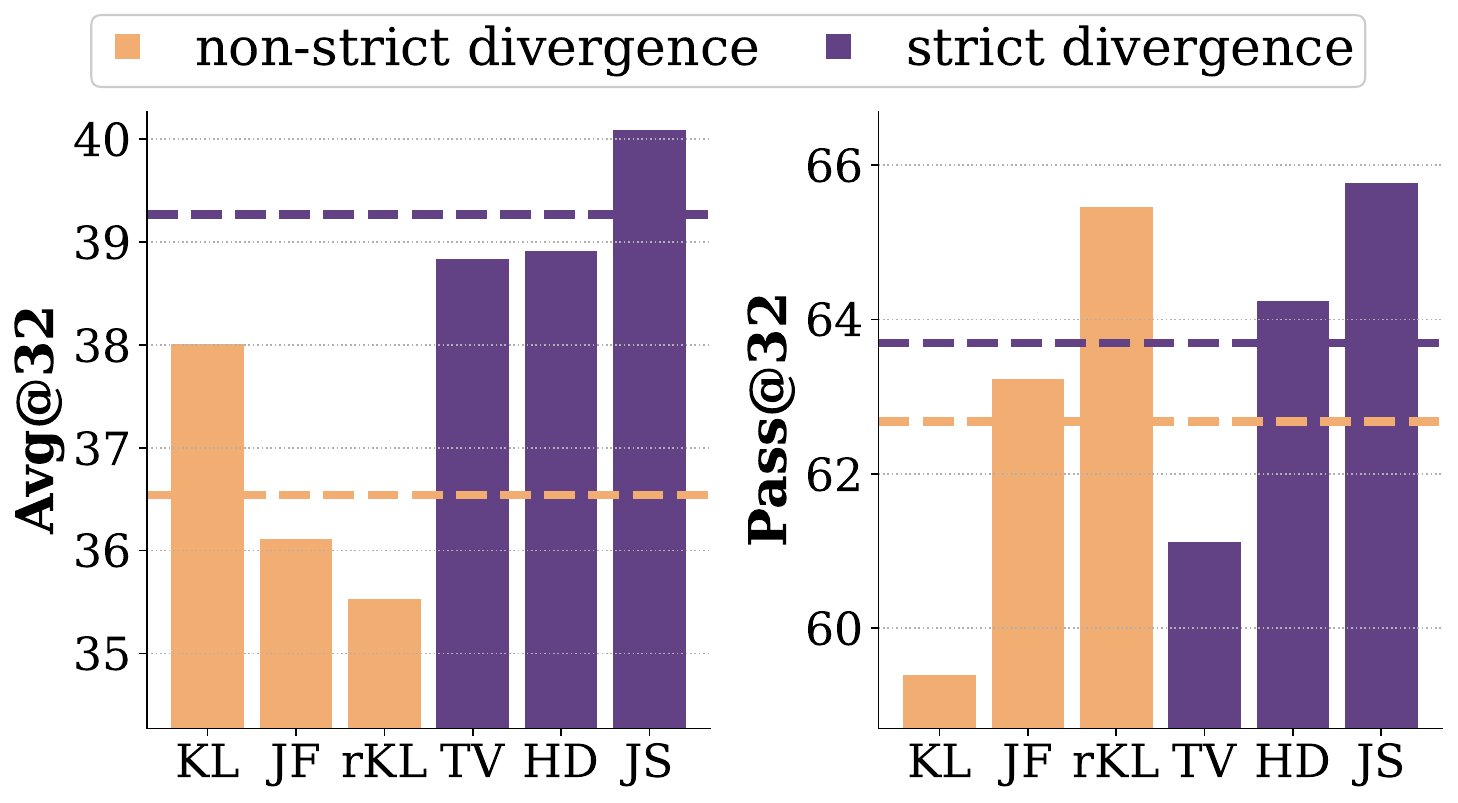}
        \vspace{-1.5em}
        \caption{The average math reasoning performance of LAD loss variants under different $f$-divergence classes. The dashed lines represent the average performances of the \textcolor{violet}{strict} and \textcolor{orange}{non-strict} divergence classes.}
        \label{fig:more-div}
    \end{minipage}
\end{figure}

\paragraph{Effect of divergence classes: strict divergences outperform weaker ones.} The math reasoning performances of all LAD variants are shown in Figure~\ref{fig:more-div}.
We comprehensively consider the KL divergence (KL), the reverse KL divergence (rKL), the Jeffreys divergence (JF), the total variation distance (TV), the Hellinger Distance (HD), and the Jensen-Shannon divergence (JS). Among all the divergence classes in our experiments, TV, HD, and JS are strict divergences that force exact matching at their minimum. On the contrary, KL, rKL, and JF are weaker divergences. Minimizing them does not necessarily imply distributional equality in practice.
Overall, strict divergences generally lead to higher overall performance because they penalize discrepancies symmetrically, thereby helping LAD learn the advantage distribution more faithfully, which in turn leads to improved reasoning outcomes.

\paragraph{Effect of $\eta$.}
To evaluate the sensitivity of LAD to $\eta$, we train models with six values of $\eta\in\{0.5,1,2,4,8,16\}$. Figure~\ref{fig:eta} reports the average Avg@16 and Pass@16 scores across six mathematical reasoning benchmarks.
These results suggest that LAD consistently outperforms GRPO under a broad range of $\eta$, indicating that its effectiveness is not critically dependent on precise tuning of $\eta$.

\begin{wrapfigure}{r}{0.5\textwidth}
    \centering    
    \includegraphics[width=0.98\linewidth]{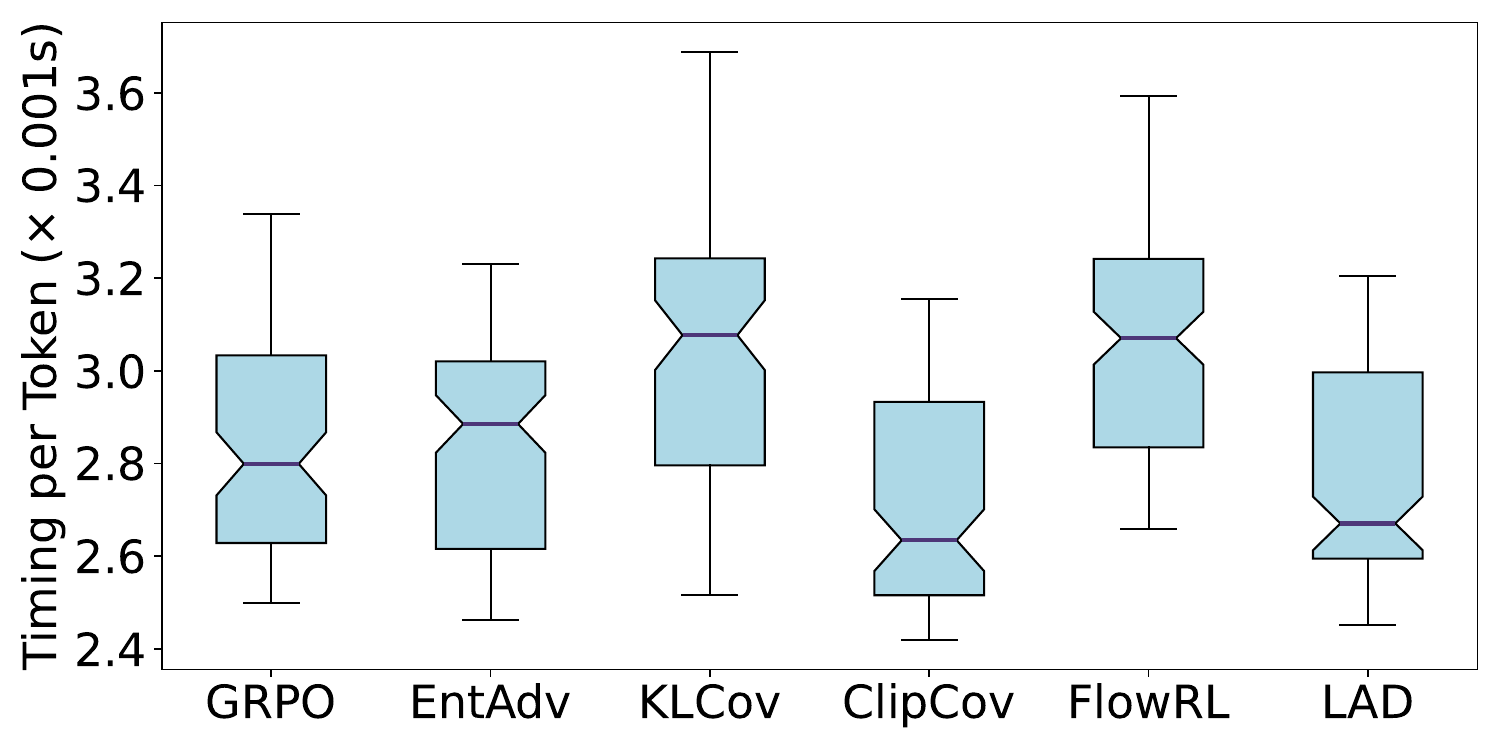}
    \caption{The distribution of GPU time per token throughout the training process.}
    \label{fig:placeholder}
\end{wrapfigure}

\paragraph{Computational cost of LAD is comparable to GRPO.}
Figure~\ref{fig:placeholder} compares the training costs of LAD with those of the baseline methods. FlowRL and KLCov incur the highest computational overhead, as both require additional evaluations of the reference model probabilities; moreover, FlowRL introduces an extra auxiliary module to approximate an otherwise intractable term in its objective. In contrast, LAD does not require any additional model components or extra forward passes. Its implementation involves only a minimal modification to the loss function, resulting in a training cost that is comparable to that of the vanilla GRPO baseline.

\section{Related Works}

Reinforcement learning has emerged as a major approach for LLM post-training \citep{qwen2.5,llama3,gemini}, which have been widely applied to various domains, including instruction following \citep{rlhf1, rlhf2} and human-value alignment \citep{dpo,ipo,simpo,iterative-dpo}. More recently, reinforcement learning has demonstrated substantial gains in enhancing the reasoning capabilities of LLMs \citep{pqm,freeprm, prime}. In this context, reinforcement learning with verifiable rewards (RLVR) has emerged as a particularly effective approach, owing to its robust and reliable reward signals \citep{rlvr-1,rlvr-2,rlvr-3,rlvr-4,rlvr-5,rlvr-6}. 
Despite these advances, most existing RL methods for LLMs continue to rely on advantage maximization as the core training objective, which inherently biases policies toward a small set of high-reward responses and limits generative diversity \citep{entropy-mechanism,diverse-vital}. 
Concurrent with our work, FlowRL \citep{flowrl} also moves beyond advantage maximization to learn reward-matching. However, as evidenced in Figure~\ref{fig:toy-dist} and Appendix~\ref{sec:specialcase}, FlowRL exhibits limited distribution matching behavior and can be interpreted as a special case of our LAD framework under stricter conditions.
Additional related work and detailed comparisons with traditional RL paradigms, including distributional RL, risk-sensitive RL and multi-objective RL, are provided in Appendix~\ref{sec:apd:related-work}.

\section{Conclusion}

Current RL objectives for large-model reasoning primarily maximize expected rewards, which encourages overfitting to dominant outcomes and suppresses viable alternative reasoning paths. We instead introduce Learning Advantage Distributions (LAD), a distribution-matching framework that optimizes policy updates by minimizing an $f$-divergence between the policy and the advantage-induced action distribution. This approach preserves multiple promising reasoning modes and avoids mode collapse without additional regularization. Experiments on multiple math and code reasoning benchmarks show that LAD consistently improves both accuracy and generative diversity over strong RLVR baselines.

\section*{Acknowledgement}
We thank Changdae Oh and Yu Wang for their valuable suggestions on the paper. This work is supported in part by the AFOSR Young Investigator Program under award number FA9550-23-1-0184, National Science Foundation under awards IIS-2237037 and IIS-2331669, Office of Naval Research under grant number N00014-23-1-2643, Schmidt Sciences Foundation, Open Philanthropy, Alfred P. Sloan Fellowship, and gifts from Google and Amazon. 

\section*{Limitations}
While LAD demonstrates promising improvements in performance and diversity, several limitations remain. First, our work focuses on post-training scenarios where responses can be evaluated with verifiable signals (e.g., mathematical or programmatic correctness). Although this setting is common in recent RL research for LLMs, extending LAD to domains without objective verifiers, such as open-ended dialogue or creative writing, may require additional mechanisms for reward signal quality. This represents an orthogonal challenge rather than a failure of the method itself.

Second, our empirical evaluations focus on mathematical and code reasoning benchmarks. These benchmarks are representative of RLVR-style settings but do not encompass all possible LLM capabilities. Broader evaluation across multimodal tasks, multilingual reasoning, and agentic decision-making is left for future work. Meanwhile, our evaluation is conducted on 1.5B- and 7B-scale models. These sizes are standard for research-stage RL studies, but we do not presume automatic transfer to 30B–70B+ or frontier-scale models. Scaling studies are left for future work and are expected to be influenced by factors such as sampling budgets, hyperparameters, and verifier reliability, rather than limitations of LAD itself.

Finally, although LAD provides theoretical justification, we do not claim formal convergence guarantees for all divergence classes or optimization schedules. This is a standard limitation shared with most large-model RL methods, and strengthening theoretical guarantees remains an interesting direction for future research rather than a missing requirement for practical validity.

\section*{Ethic Considerations}
This work focuses on reinforcement learning for reasoning tasks, where correctness can be automatically evaluated. As such, our experiments are conducted exclusively in domains with objective ground-truth verification to reduce risks associated with generating misleading or harmful content. 

Our approach of learning advantage distribution (LAD) introduced here operate purely at the optimization and distribution-learning level and do not involve new data sources, new modalities, or sensitive information. In deployment settings, LAD should be paired with standard safety practices for LLMs (e.g., alignment checks, content filtering, and supervised oversight), which are orthogonal to the contributions of this work.

\bibliography{custom}

\appendix

\section{Related Works}
\label{sec:apd:related-work}
\subsection{Traditional RL beyond Expected-Scalar Maximization}
Beyond the standard formulation of reinforcement learning (RL) as maximizing expected cumulative scalar reward or advantage, a substantial body of work has explored alternative objectives that encode richer behavioral preferences.

\paragraph{Distributional RL.} 
Distributional reinforcement learning (RL) \citep{distrl0,distrl-1,distrl-2,distrl-3,distrisk} departs from expectation-based value estimation by modeling the full distribution of returns and optimizing distributional objectives, thereby enabling finer-grained control over outcome profiles. Despite the similarity in terminology, distributional RL and learning the advantage distribution (LAD) address fundamentally different problems. In distributional RL, the learned distribution $Z^\pi(x,y)$ captures the stochasticity of returns for a fixed state–action pair, reflecting randomness induced by the environment and transition dynamics. The return distribution thus serves as an intermediate representation to improve value estimation, while policy optimization remains expectation-based through $\mathbb{E}_{\pi}[Z^\pi(x,y)]$. In contrast, LAD models the distribution over the entire action space induced by advantage values. This distribution appears directly in the learning objective, which minimizes a divergence of the form $\min_\pi \mathbb{D}_f(\mathcal{P_A} \Vert \mathcal{P_\pi})$, and therefore plays a central role in policy optimization rather than serving as an auxiliary estimator.

\paragraph{Risk-sensitive RL.} 
Risk-sensitive RL \citep{risk1,risk2,risk3,distrisk,risk4} extends the standard RL objective by optimizing criteria that reflect attitudes toward uncertainty rather than mean performance alone. In many classical formulations \citep{riskvar-1,riskvar-2,riskvar-3}, risk sensitivity is introduced by augmenting the expected return with an additional term that penalizes variability or tail risk. More recent approaches include variance-constrained actor–critic methods and distributional techniques that explicitly model the return distribution and optimize risk-aware criteria such as conditional value-at-risk (CVaR) within that framework \citep{cvar1,cvar2}. Nevertheless, similar to distributional RL, the distributions considered in risk-sensitive RL are induced by environmental stochasticity, and policy optimization typically retains an expectation-based form, albeit with a modified objective reflecting risk preferences.

\paragraph{Multi-objective RL.}
Multi-objective reinforcement learning (MORL) \citep{multirl0, multirl1, multirl2, multirl3,multirl4,multirl5} generalizes the standard RL framework by considering vector-valued reward functions and learning policies that optimize trade-offs among multiple, potentially conflicting objectives. Common solution concepts include learning Pareto-optimal policies~\citep{pareto1,pareto2}, approximating Pareto fronts \citep{paretoapp1}, or optimizing scalarized objectives defined by linear or nonlinear combinations of the individual reward components \citep{scalarmorl,scalarmorl2}. In these approaches, the multiple objectives are typically specified exogenously and correspond to distinct semantic criteria (e.g., performance, cost, safety), and policy optimization proceeds by reducing the vector-valued returns to a scalar objective through scalarization or preference weights. In contrast, learning the advantage distribution (LAD) does not assume multiple predefined reward dimensions. Instead, LAD induces a distribution over actions from scalar advantage values and directly matches this distribution in the learning objective. As a result, LAD modifies the policy objective through distributional alignment over actions rather than through explicit multi-criteria optimization or scalarization of vector-valued rewards.

\subsection{Exploration-exploitation trade-off in RL for LLMs}

The exploration–exploitation trade-off \citep{ee-trade} is a foundational challenge in reinforcement learning. In the context of large language models (LLMs), most existing work has focused on improving exploitation through increasingly sophisticated objectives. However, a growing body of research argues that insufficient exploration substantially limits the effectiveness of current RL-based fine-tuning methods for LLMs, motivating explicit mechanisms to promote exploration during training.

\paragraph{Optimistic exploration in RLHF.}
Several recent works \citep{xpo,onlinestrength-onpolicy} in reinforcement learning from human feedback (RLHF) emphasize the importance of deliberate exploration during policy optimization. These approaches introduce structured exploration mechanisms to mitigate mode collapse and premature convergence, including modifying sampling strategies during rollouts \citep{xpo,online-sampler} and selecting training samples based on uncertainty estimates \citep{iterative-dpo,prompt-1,prompt-2}. Other methods explicitly encourage exploration by optimizing objectives beyond immediate reward maximization, such as incorporating uncertainty-aware rewards \citep{sea} or adding exploration bonuses derived from disagreement or novelty signals \citep{selm,vpo,geb}. Collectively, these works demonstrate that, in the LLM setting, exploration is not merely a training stability issue but a critical factor for discovering high-quality and diverse generations under sparse, noisy, or biased reward signals. While effective, many of these methods rely on computationally expensive uncertainty estimation or auxiliary models. In contrast, LAD promotes sustained exploration by reformulating the advantage maximization objective as a distribution-matching objective, naturally maintaining higher policy entropy throughout training without additional uncertainty modeling.

\paragraph{Entropy collapse and regulation in RLVR.}
A related line of work \citep{entropy-mechanism,decompose-entropy,sustain} identifies rapid entropy collapse in the early stages of reinforcement learning with verifiable rewards (RLVR), motivating approaches that encourage exploration by augmenting advantage maximization with entropy-based regularization \citep{entropy-in-adv,8020,ent-1,ent-2}. However, these methods retain advantage maximization as the primary objective, limiting the extent to which entropy regularization can counteract exploitation-dominated updates. Moreover, entropy regularization generally alters the optimization landscape and may shift the optimum away from the original policy improvement objective. In contrast, LAD directly reformulates the learning objective by shifting from advantage maximization to learning the advantage distribution, while preserving the same optimal policy implied by trust-region policy optimization. This reformulation enables principled exploration without introducing auxiliary regularizers or modifying the target optimum.

\section{Analyses \& Proofs}

\subsection{Proofs of Lemma~\ref{lemma:two-distribution}}
\textbf{Lemma 3.1.} \textit{Let $\pi_\theta$ denote the current policy parameterized by $\theta$, and let $\pi_{\mathrm{old}}$ be the behavior model. Suppose that $\pi_\theta$ is optimized using a trust-region constrained reinforcement learning algorithm (e.g., PPO \citep{ppo}) as in Eq.~\ref{eq:trpo}. There exists a Lagrange multiplier $\eta>0$, such that the policy update satisfies
    \begin{equation}
        \frac{\frac{\pi_\theta(y|x)}{\pi_\mathrm{old}(y|x)}}{Z_\pi(x)} = \frac{e^{\frac{A(x,y)}{\eta}}}{Z_A(x)} ,
    \end{equation}
where $\eta$ is a Lagrangian multiplier, $Z_\pi(x) = \sum_y \frac{\pi_\theta(y|x)}{\pi_\mathrm{old}(y|x)}$ and $Z_A(x)=\sum_y e^{\frac{A(x,y)}{\eta}}$ are two normalizations.}

\begin{proof}
Consider one step of policy improvement that maximizes the expected advantage under the new policy while staying close to the old policy $\pi_\mathrm{old}$ (trust-region). For each prompt $x$, we can formulate a constrained maximization problem as 
\begin{align}
    \max_{\pi_\theta} \sum_y A(x,y)\pi_\theta(y|x)&\quad \mathrm{s.t.} \quad\sum_y \pi_\theta(y|x) =1\nonumber\\ \mathrm{and} \quad D_\mathrm{KL}(\pi_\theta(\cdot|x)&\Vert\pi_\mathrm{old}(\cdot|x))\leq \epsilon   .
\end{align}
We can form the per-state Lagrangian with multipliers $\eta\geq 0$ and $\lambda$ as follows,
\begin{align}
    L =& \sum_y A(x,y)\pi_\theta(y|x)  + \lambda(\sum_y \pi_\theta(y|x) - 1)\nonumber\\
    &- \eta\sum_y(\pi_\theta(y|x) \log \frac{\pi_\theta(y|x)}{\pi_\mathrm{old}(y|x)}) - \epsilon.
\end{align}
Stationarity w.r.t. $\pi_\theta$ gives
\begin{equation}
A(x,y)-\eta(\log\frac{\pi_\theta(y|x)}{\pi_\mathrm{old}(y|x)}+1)+\lambda=0,
\end{equation}
which equals $\log\frac{\pi_\theta(y|x)}{\pi_\mathrm{old}(y|x)}= \frac{A(x,y) + \lambda-\eta}{\eta}$.
With some algebra, we have
\begin{equation}
    \frac{\frac{\pi(y|x)}{\pi_\mathrm{old}(y|x)}}{Z_\pi(x)} = \frac{e^{\frac{A(x,y)}{\eta}}}{Z_A(x)} ,
\end{equation}
where $Z_\pi(x) = \sum_y \frac{\pi(y|x)}{\pi_\mathrm{old}(y|x)}$ and $Z_A(x)=\sum_y e^{\frac{A(x,y)}{\eta}}$ are two normalizations.
\end{proof}

\subsection{Proofs of Lemma~\ref{lemma:equivalence}}
\textbf{Lemma 2.}  \textit{Let $\mathcal{L}= \mathbb{E}_{x\sim\rho} \sum_y g_2(x,y) f(\frac{p_\theta(y|x)}{ g_1(x,y)})$, where $p_\theta(\cdot)$ is a probability measure with $\sum_y p(\pi)=1$. $g_1(x,y),g_2(x,y)$ are two functions independent to $p_\theta$. $p_\theta(y|x)$ always satisfy $p_\theta(y|x) \propto g_1(x,y)$ as long as the fraction $\frac{g_2(x,y)}{g_1(x,y)} = Z_g(x)$ is irrespective to $y$.}

\begin{proof}
Similar to Lemma~\ref{lemma:two-distribution}, we adopt the same constrained maximization problem with Lagrange multipliers, and the stationarity condition can be formulated as
\begin{equation}
    \frac{g_2(x,y)}{g_1(x,y)} f'(\frac{p(\pi(y|x))}{g_1(x,y)}) = \lambda,
\end{equation}
where $\lambda$ is a Lagrange multiplier of $\sum_y p(\pi(y|x))=1$. Since $f'$ is an increasing function according to the definition in Definition~\ref{def:f-div}, we have for a fixed state $x$, $\forall y, \frac{p(\pi(y|x))}{g_1(x,y)}$ is a constant, thus $p_\theta(y|x)\propto g_1(x,y)$.
\end{proof}

\subsection{LAD is an $\mathcal{O}(\delta)$-accurate surrogate of the theoretical LAD objective}
\label{sec:apd:theorem}
First, we made some standard assumptions for the following proofs.
\begin{assumption}[finite actions with uniform lower bound $p_\mathrm{min}$]
    For any policy $\pi$ and any state-action pair $x,y$, $1 \geq p_\mathrm{max} \geq \pi(y|x) \geq p_\mathrm{min} >0$.   
\label{assump:p-min}
\end{assumption}
\begin{assumption}[bounded advantage]
For any state-action pair $(x,y)$, the advantage value is bounded by a small value $\vert A(x,y)\vert < \gamma $
\label{assump: bound-adv}
\end{assumption}
\begin{assumption}[small update per step]
Under a small learning rate, we assume that for any $(x,y)$, $\Big\vert \frac{\pi(y|x)}{\pi_\mathrm{old}(y|x)} -1  \Big\vert  < \delta$
\label{assump: small-step}
\end{assumption}
The first assumption is generally established in the LLM setting due to a finite vocabulary size.
For the second assumption, the advantage scale is controllable by constraining the reward scale and normalization, which are generally adopted in existing implementations like GRPO~\citep{grpo}.
For the third assumption, since the sampling distribution $\pi_\mathrm{old}$ is frequently updated (generally within ten steps), the overall distribution of $\pi$ would not have a sharp change compared to $\pi_\mathrm{old}$. Hence, Assumption~\ref{assump: small-step} also holds empirically, especially considering that standard RL implementations for LLMs generally have small learning rates and gradient clipping mechanisms.

With the above assumptions, we obtain the following conditions, 
\begin{equation}
    \Big\vert \frac{Z_A(x)}{Z_\pi(x)} - 1 \Big\vert < \frac{\mathcal{O}(\gamma+\delta)}{1+ \mathcal{O}(\delta)}
\end{equation}
Holding this inequality, we can then prove that 
the practical LAD loss in Eq.~\ref{eq:practical} is an $\mathcal{O}(\delta)$-accurate surrogate of the theoretical LAD objective in Eq.~\ref{eq:theoretical}.

\begin{theorem}[LAD is an $\mathcal{O}(\delta)$-accurate surrogate of the theoretical LAD objective]
Under Assumptions \ref{assump:p-min},\ref{assump: bound-adv},\ref{assump: small-step}, the gap between the theoretical LAD objective $\mathcal{L}_\mathrm{LAD}^\mathrm{theorem}$ and the practical LAD objective  $\mathcal{L}_\mathrm{LAD}$ is bounded by
\begin{equation}
    \Big\vert  \mathcal{L}_\mathrm{LAD}^\mathrm{theorem} - \mathcal{L}_\mathrm{LAD} \Big\vert  < {\mathcal{O}(\gamma+\delta)}
\end{equation}
\label{thm:surrogate-formal}
\end{theorem}

\begin{proof}
In the proof, we denote $c_\theta = \frac{\pi_\theta(y|x)}{\pi_\mathrm{old}(y|x)}\cdot e^\frac{-A(x,y)}\eta$ for simplicity.
First, we re-express the expectations in $\mathcal{L}_\mathrm{LAD}^\mathrm{theorem}$ and $\mathcal{L}_\mathrm{LAD}$ with respect to a common base distribution $\mu_x$, and absorb the differences into weights.
\begin{align}
  \mathcal{L}_\mathrm{LAD}^\mathrm{theorem} &= \mathbb{E}_{x\sim\rho,y\sim \mu_x}  \frac{e^{\frac{A(x,y)}{\eta}}}{\mu_x(y) Z_A(x)} f(c_\theta \frac{Z_A(x)}{Z_{\pi_\theta}(x)}),\nonumber\\
  \mathcal{L}_\mathrm{LAD} &= \mathbb{E}_{x\sim\rho,y\sim\mu_x}  \frac{e^\frac{A(x,y)}{\eta}}{\mu_x(y)}\pi_\mathrm{old}(y|x) f(c_\theta).\nonumber
\end{align}
Now, both objectives are written as expectations over the same uniform distribution $\mu_x=1/N$, where $N$ is the number of all possible actions. Then we can have the following subtraction,
\begin{align}
    \Big\vert  \mathcal{L}_\mathrm{LAD}^\mathrm{theorem} - \mathcal{L}_\mathrm{LAD} \Big\vert \leq  \mathbb{E}_{x\sim\rho,y\sim\mu_x} \Big\vert \frac{e^\frac{A(x,y)}{\eta}}{\mu_x(y)} F(x,y) \Big\vert \nonumber\\
    F(x,y)\!=\!\frac{1}{Z_A(x)}f(\frac{c_\theta  Z_A(x)}{Z_{\pi_\theta}(x)}) -  \pi_\mathrm{old}(y|x) f(c_\theta) \nonumber
\end{align}
where we use the triangle inequality here. With Assumption~\ref{assump:p-min} and Assumption~\ref{assump: bound-adv}, for any $(x,y)$, we have
\begin{align}
    \Big\vert\frac{e^\frac{A(x,y)}{\eta}}{\mu_x(y)} F(x,y)\Big\vert \leq  C  \Big\vert f(c_\theta \cdot\frac{Z_A(x)}{Z_{\pi_\theta}(x)}) -  f(c_\theta) \Big\vert \nonumber
\end{align}
where $C= e^\frac{\gamma}\eta \cdot N  \cdot \max(p_\mathrm{max}, \frac{e^\frac{\gamma}\eta }{N})$ is a constant. Next, we write the Taylor expansion of $f(c_\theta \cdot\frac{Z_A(x)}{Z_{\pi_\theta}(x)})$ to bound the right term.
\begin{equation}
    f(c_\theta \cdot\frac{Z_A(x)}{Z_{\pi_\theta}(x)}) = f(c_\theta) + f'(\zeta)c_\theta(\frac{Z_A(x)}{Z_{\pi_\theta}(x)} - 1),\nonumber
\end{equation}
where $\zeta$ is a value between $c_\theta$ and $c_\theta \frac{Z_A(x)}{Z_{\pi_\theta}(x)}$.
Hence, for any $(x,y)$, we can obtain
\begin{align}
   \Big\vert  \mathcal{L}_\mathrm{LAD}^\mathrm{theorem} - &\mathcal{L}_\mathrm{LAD} \Big\vert \nonumber\\& \leq \frac{C}{N} \mathbb{E}_{x\sim\rho} \sum_y \Big\vert f'(\zeta)c_\theta(\frac{Z_A(x)}{Z_{\pi_\theta}(x)} - 1) \Big\vert \nonumber
\end{align}
Thus, we have
\begin{equation}
    \Big\vert  \mathcal{L}_\mathrm{LAD}^\mathrm{theorem} - \mathcal{L}_\mathrm{LAD} \Big\vert  < {\mathcal{O}(\gamma+\delta)}\nonumber
\end{equation}
\end{proof}

\begin{figure}[t!]
    \centering
    \includegraphics[width=0.49\linewidth]{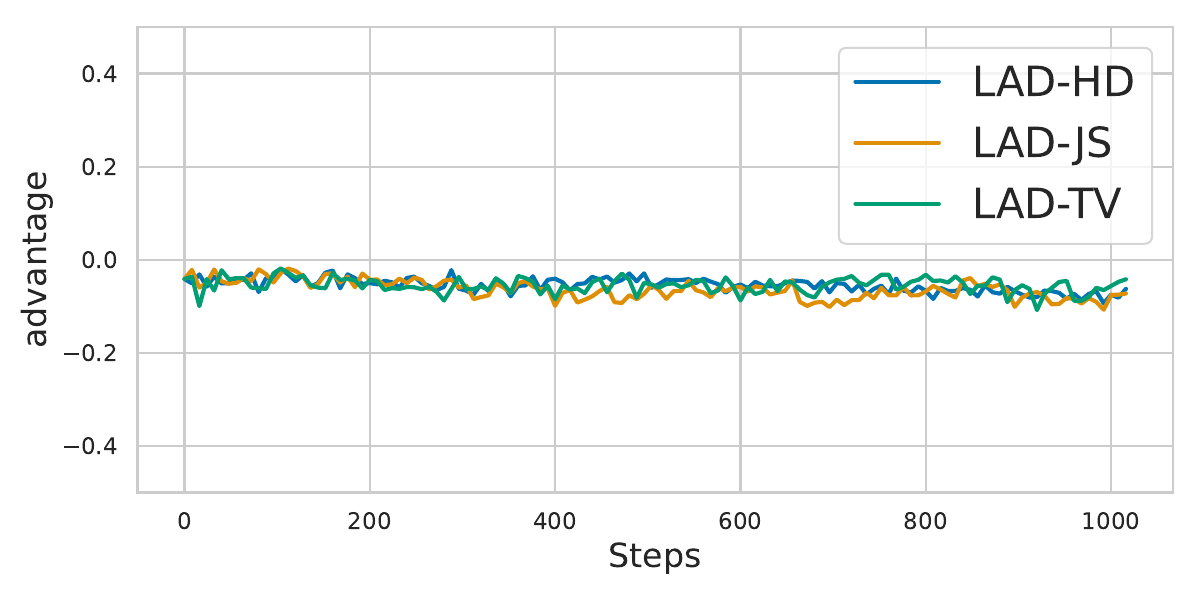}
     \includegraphics[width=0.49\linewidth]{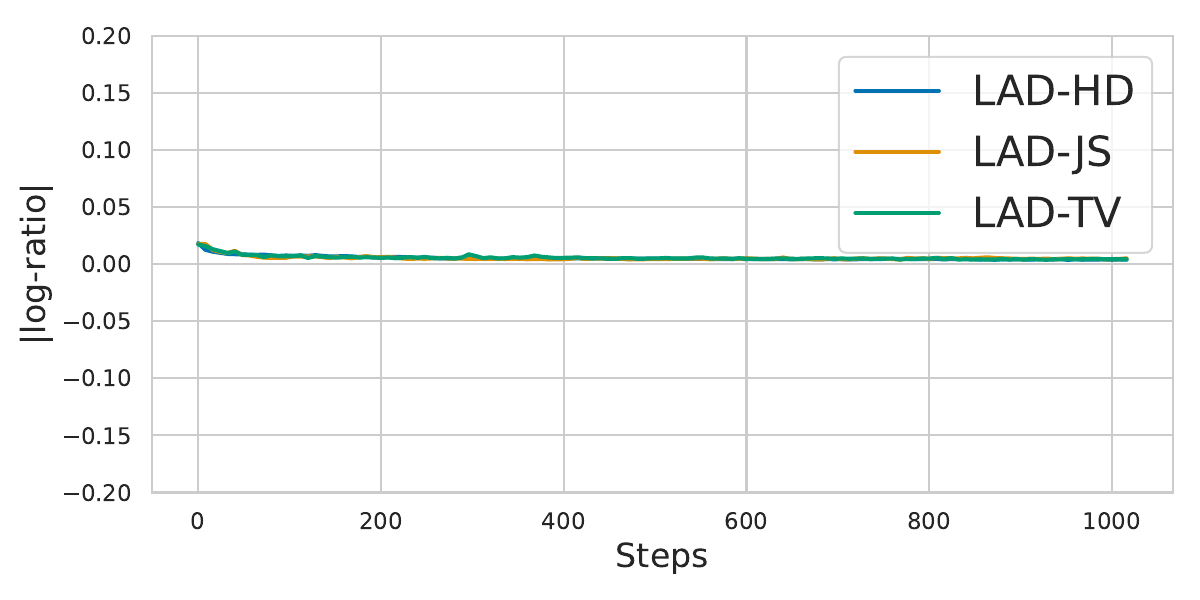}
    \caption{The average advantage value and the log-ratio of $\vert \log \frac{\pi_\theta(y|x)}{\pi_\mathrm{old}(y|x)} \vert $ throughout the training.} 
    \label{fig:validation}
\end{figure}

\paragraph{The empirical validation for assumptions}
To empirically validate our assumption hold true, we show the average advantage value and the log-ratio of $\vert \log \frac{\pi_\theta(y|x)}{\pi_\mathrm{old}(y|x)} \vert $ during the training process of three LAD variants under strict divergence classes. As shown in Figure~\ref{fig:validation}, advantage values are empirically bounded by a small value, i.e., $\gamma$ in Assupmtion~\ref{assump: bound-adv} is small. Meanwhile, the $\frac{\pi_\theta(y|x)}{\pi_\mathrm{old}(y|x)}$ is indeed near 1, i.e., $\epsilon$ in Assumption~\ref{assump: small-step} is close to 0. Therefore, the bound of $\Big\vert \mathcal{L}_\mathrm{LAD}^\mathrm{theorem} -\mathcal{L}_\mathrm{LAD} \Big\vert$ as proved in Theorem~\ref{thm:surrogate-formal} is small, which indicates that $\mathcal{L}_\mathrm{LAD} $ is indeed a close surrogate of $\mathcal{L}_\mathrm{LAD}^\mathrm{theorem}$.

\subsection{Gradient Comparisons}
\label{sec:apd:gradient-comparison}
Beyond the main paper’s comparison between the gradients of the LAD loss and the vanilla policy gradient, we further analyze how the LAD gradient differs from both the PPO gradient and the theoretical LAD formulation.

In practice, PPO enforces trust-region constraints through a clipped surrogate objective. The resulting gradient can be written as
\begin{equation}  \mathbb{E}_{x\sim\rho,y\sim\pi_\theta(\cdot|x)}\!\!\left[\!-\mathrm{sg}(\frac{\pi_\theta}{\pi_\text{old}}, \!\epsilon) \!A(x,\!y)\!\nabla_\theta\! \log \pi_\theta(y|x) \! \right],  
\nonumber
\end{equation}
where $\mathrm{sg}(\frac{\pi_\theta}{\pi_\text{old}},\epsilon)$ is a binary signal function that equals $1$ when $(A(x,y)>0\land\frac{\pi_\theta}{\pi_\text{old}}\leq 1+\epsilon) \lor (A(x,y)<0\land\frac{\pi_\theta}{\pi_\text{old}}\geq 1-\epsilon)$, and $0$ otherwise.
Compared with the vanilla policy gradient in Eq.~\eqref{eq:reinforce}, PPO introduces an additional gating mechanism that suppresses excessively large updates. However, once the gate is active, the update direction remains proportional to the advantage. As a result, PPO still predominantly increases the likelihood of responses with larger advantages, reinforcing already high-probability actions. Although this clipping improves stability, it does not fundamentally alter the advantage-maximization nature of the objective, and therefore still tends to reduce generative diversity in a manner similar to the vanilla policy gradient.
For the theoretical LAD objective, the gradient takes the following form
$\nabla_\theta \mathcal{L}_\mathrm{LAD}^\mathrm{theorem} $ is 
\begin{align}
    &\mathbb{E}_{x\sim\rho,y\sim \mathcal{P}_{\pi_\theta}(\cdot|x)} f'(\frac{\pi_\theta(y|x)}{\pi_\mathrm{old}(y|x)}\cdot e^\frac{-A(x,y)}\eta \cdot \frac{Z_A(x)}{Z_\pi (x)}) \cdot \nonumber\\&\quad (\nabla_\theta \log \pi_\theta(y|x) - \mathbb{E}_{y'\sim \mathcal{P}_{\pi_\theta}(\cdot|x)} \nabla_\theta \log \pi_\theta (y'|x) ),\nonumber
\end{align}
where $\mathcal{P}_{\pi_\theta}$ is the policy-induced distribution defined in Eq.~\ref{eq:pi}.
Compared to the gradient of the practical LAD loss, this gradient includes an explicit centering term that subtracts the expected score function under the policy-induced distribution. This structure enforces competition among actions and prevents any single response from dominating the update solely due to a large advantage.
Nevertheless, the practical LAD loss retains the core behavior of the theoretical gradient while removing intractable normalization terms. Compared with the theoretical LAD gradient, the practical LAD gradient omits the expectation-centering term but preserves the crucial dependence on both the advantage and the likelihood ratio. As a result, practical LAD continues to suppress updates for actions whose probabilities are already sufficiently large, even when their advantages remain high. This mechanism fundamentally differs from advantage maximization and enables LAD to maintain stable optimization while preserving generative diversity.

\subsection{FlowRL is a special case in our framework, but with stricter conditions}
\label{sec:specialcase}
FlowRL \citep{flowrl} proposes to learn a reward-matching objective by directly regressing the log policy ratio onto advantage values using a mean-squared-error loss inspired by GFlowNets~\citep{gflownet}. In contrast, LAD provides a more general and theoretically grounded framework, within which FlowRL arises as a special case under restrictive assumptions on policy updates and parameter scaling.
Specifically, under a small-step learning regime where the updated policy $\pi_\theta$ remains close to the behavior policy $\pi_\mathrm{old}$, the likelihood ratio $\frac{\pi_\theta(y|x)}{\pi_\mathrm{old}(y|x)}$ stays near 
$1$. If we further assume that this ratio satisfies $\frac{\pi_\theta(y|x)}{\pi_\mathrm{old}(y|x)}>\frac{1}{e}$, then the function $f(x) = x(\log x)^2$ convex over the relevant domain $x>\frac{1}{e}$. The convexity ensures that the corresponding $f$-divergence objective is well-defined. Further, when $\eta\rightarrow\infty$, $e^\frac{A}{\eta} \rightarrow 1+\frac{A}{\eta}$. Therefore, for each $x\sim\rho$, the LAD loss can be formulated as 
\begin{equation}
\mathbb{E}_{y\sim\pi_\theta(\cdot|x)}  [ \log Z(x) + \log \frac{\pi_\theta(y|x)}{\pi_\mathrm{old}(y|x)} - \frac{A(x,y)}\eta ]^2,
\label{eq:specialcase}
\end{equation}
where $Z(x) = \frac{Z_A(x)}{eZ_{\pi_\theta}(x)}$.
Eq.~\ref{eq:specialcase} is exactly the objective optimized by FlowRL, showing that FlowRL can be viewed as a special instantiation of LAD.

\section{Experiments}

\begin{table*}[th!]
\renewcommand{\arraystretch}{2}
    \centering
    \caption{LAD losses under different divergences. The $*$ represents the corresponding divergence class is a strict divergence that forces exact matching at its minimum.}
    \label{tab:lad-losses}
    \small
    \scalebox{0.9}{
    \begin{tabular}{cccc}
    \toprule
       Divergence  & $f(x)$ & $\mathcal{L}_\mathrm{LAD}$ \\
       \midrule
        KL divergence & $x\log x$ &$ \mathbb{E}_{x\sim\rho,y\sim\pi_\mathrm{old}(\cdot|x)}  \frac{\pi_\theta(y|x)}{\pi_\mathrm{old}(y|x)} \Big( \log\frac{\pi_\theta(y|x)}{\pi_\mathrm{old}(y|x)} - \frac{A(x,y)}\eta \Big) $ \\
        reverse KL divergence & $-\log x$ & $\mathbb{E}_{x\sim\rho,y\sim\pi_\mathrm{old}(\cdot|x)} - {e^\frac{A(x,y)}{\eta}} \log {\pi_\theta(y|x)}$ \\
        Jeffreys divergence & $(x-1)\log x$ & $ \mathbb{E}_{x\sim\rho,y\sim\pi_\mathrm{old}(\cdot|x)}   (\frac{\pi_\theta(y|x)}{\pi_\mathrm{old}(y|x)}-e^\frac{A(x,y)}{\eta}) \log [\frac{\pi_\theta(y|x)}{\pi_\mathrm{old}(y|x)}e^\frac{-A(x,y)}{\eta}]$ \\
        Total Variance Distance$^*$& $|x-1|$ & $\mathbb{E}_{x\sim\rho,y\sim\pi_\mathrm{old}(\cdot|x)}  {e^\frac{A(x,y)}{\eta}} \Big\vert \frac{\pi_\theta(y|x)}{\pi_\mathrm{old}(y|x)}\cdot e^\frac{-A(x,y)}\eta -1 \Big\vert $\\
        Hellinger Distance$^*$& $\frac{1}{2}(\sqrt t -1)^2$ & $\mathbb{E}_{x\sim\rho,y\sim\pi_\mathrm{old}(\cdot|x)}  {e^\frac{A(x,y)}{\eta}} ( \sqrt{\frac{\pi_\theta(y|x)}{\pi_\mathrm{old}(y|x)}}\cdot e^\frac{-A(x,y)}{2\eta} -1 )^2 $\\
        Jensen–Shannon divergence$^*$& $\frac{1}{2}(x\log x -(x+1)\log\frac{x+1}{2})$&
      { $\begin{aligned}\scriptstyle
\mathbb{E}_{x\sim\rho,y\sim\pi_\mathrm{old}(\cdot|x)}  {e^\frac{A(x,y)}{\eta}} \Big( \frac{\pi_\theta(y|x)}{\pi_\mathrm{old}(y|x)}e^\frac{-A(x,y)}\eta \log(\frac{\pi_\theta(y|x)}{\pi_\mathrm{old}(y|x)}\\ \scriptstyle \cdot e^\frac{-A(x,y)}\eta)   - (\frac{\pi_\theta(y|x)}{\pi_\mathrm{old}(y|x)}e^\frac{A(x,y)}\eta +1)\log (\frac{\pi_\theta(y|x)}{2\pi_\mathrm{old}(y|x)}e^\frac{A(x,y)}\eta)  \Big) \nonumber
        \end{aligned}$} \\
        \bottomrule
    \end{tabular}}
\end{table*}

\begin{table*}[t!]
    \centering
     \caption{The detailed math reasoning performances of LAD with different hyperparameters.}
    \label{tab:append-eta-ablation}
    \scalebox{0.8}{
    \begin{tabular}{cc|ccccccc}
    \toprule
      Methods & $1/\eta$ & MATH & \makecell[c]{Olympiad\\Bench}& \makecell[c]{Minerva}& \makecell[c]{AIME\\2024}&\makecell[c]{AIME\\2025}& AMC & Average \\
      \midrule
      & & \multicolumn{7}{c}{Avg@16}\\
      GRPO& - &75.71 &39.24 &36.63& 13.89& 4.90& 53.39& 37.29\\
    
      \midrule
        
        \multirow{7}{*}{LAD-JS}  & 1/4 & 75.24 & 38.83& 35.59&15.68& 5.70&50.34 & 36.90 \\
        &1/2& 76.20& 38.02& 38.07& 18.15& 6.73 & 53.08& 38.38 \\
        &1& 74.84& 37.16& 39.84& 14.95&11.28 & 55.98& 39.00\\
         &2 & 76.90& 36.50& 34.93& 18.67& 9.13& 55.12& 38.54 \\
        & 4 &77.60 &41.29 &36.81 &19.59 &8.71 &56.37& 40.06 \\
        & 8 & 77.15& 42.06& 35.64& 14.54& 12.11& 49.34& 38.47 \\
        & 16 & 74.89& 39.30& 36.01& 18.54& 12.40& 51.60& 38.79\\
        &32& 76.47& 40.02& 34.74& 14.15& 9.33& 51.08& 37.63\\  
        \midrule
        \midrule
        & & \multicolumn{7}{c}{Pass@16}\\
        GRPO & - &89.80& 57.78& 53.31& 27.60& 23.23 &74.92 &54.44\\

        \midrule

        \multirow{7}{*}{LAD-JS}& 1/4& 92.20& 64.44& 55.88& 34.38& 25.83& 79.82& 58.76\\
        &1/2& 91.60& 61.93& 58.46& 35.63& 26.67& 80.16& 59.08\\
        &1& 87.4& 53.78& 55.15& 35.42& 32.92& 78.65& 57.22\\
        &2& 91.40& 57.33& 57.35& 40.42& 30.42& 80.91&59.64\\
        &4& 92.40& 64.15 &56.62& 40.63 &30.73 &84.30 &61.47\\
        &8& 93.00& 68.15& 52.57& 34.58& 31.04& 79.78& 59.85   \\
        &16 & 91.60& 65.33& 56.99& 37.50& 29.79& 80.08& 60.22 \\ 
        &32& 94.40& 65.63& 57.72& 38.13& 29.79& 81.36& 61.17\\ 
        \bottomrule
    \end{tabular}}
\end{table*}

\renewcommand{\arraystretch}{0.9}
\setlength{\tabcolsep}{9pt}
\begin{table*}[t!]
    \centering
      \caption{Pass@32 performance of different losses on math reasoning benchmarks. The best two results under a specific metric are highlighted in \textbf{bold} and \underline{underlined}.}
    \label{tab:main-pass}
    \scalebox{0.9}{
    \begin{tabular}{c|ccccccc}
    \toprule
        Methods  & MATH & \makecell[c]{Olympiad\\Bench}& \makecell[c]{Minerva}& \makecell[c]{AIME\\2024}&\makecell[c]{AIME\\2025}& AMC& Average  \\
        \midrule
       
        {GRPO}{\footnotesize \citep{grpo}} &90.60&60.15&55.88&30.52&26.77&80.50 & 57.40\\
        EntAdv{\footnotesize \citep{entropy-in-adv}} & 90.60&62.51&55.51&35.83&34.27&81.93&60.10 \\
        KLCov{\footnotesize \citep{entropy-mechanism}} & 92.80& 64.15& 56.99& 36.25& \textbf{41.25}& \underline{85.99} & 63.04 \\
        ClipCov{\footnotesize \citep{entropy-mechanism}} & 91.40 & 61.48 & \textbf{60.29}&37.40 & 32.08 & 80.61& 60.54\\
       {FlowRL}{\footnotesize \citep{flowrl}} &\textbf{94.20}&\textbf{69.04}&59.56&\underline{37.50} & 35.52&82.94 & \underline{63.13} \\
        \midrule
       {LAD} &\underline{93.80}&\underline{67.70} & \underline{59.93}&\textbf{47.19}&\underline{37.08}&\textbf{88.86}& \textbf{65.76}\\
        \bottomrule
    \end{tabular}}
\end{table*}

\subsection{Loss Instances}

All the instances involved in our experiments are listed in Table~\ref{tab:lad-losses}. It reveals that LAD loss is highly extensible to different divergence classes. As the star sign indicated, Jensen–Shannon divergence, Hellinger Distance, and Total Variance Distance are strict divergences, while KL divergence, the reverse KL divergence, and Jeffreys divergence are not strict divergences. As analyzed in \S\ref{sec:exp-further-study}, the strict divergence classes are recommended, which generally yield better overall performance.

\subsection{Implementation Details}
\label{sec:apd:implement}

\paragraph{Math reasoning.}
The experiments on math reasoning tasks are all conducted on two NVIDIA H200 140GB. We train our models on the DAPO-MATH dataset \citep{dapo}, following the experimental protocol of \citet{entropy-mechanism}. We use Qwen2.5-7B~\citep{qwen2.5} as our LLM backbone, and Jensen-Shannon divergence to implement LAD loss in the main experiments. At each rollout step, we sample 8 responses with a max length of 8192 tokens per prompt for a batch of 256 prompts using a sampling temperature of 1, and subsequently perform 8 policy update steps using the collected trajectories with the learning rate 1e-6. We apply the dynamic sampling strategy and the clip-higher technique from DAPO \citep{dapo} uniformly across all baselines to ensure a fair comparison. Evaluation is conducted on a diverse set of mathematical reasoning benchmarks, including MATH500 \citep{prm800k}, AIME 2024, AIME 2025 \citep{aime}, AMC \citep{aime}, OlympiadBench \citep{olympiadbench}, and Minerva \citep{minerva}. These settings ensure direct comparability with prior work. We follow \citet{flowrl} to assess logical diversity using an LLM-as-a-Judge framework with GPT-4-Turbo. Specifically, we generate 16 responses per question on AIME 2024/2025 and use the same evaluation prompt as shown in Box~\ref{box:gptprompt}. For the two 1.5B models, we change the max lengths to 4096 and keep other settings the same. For the ablation studies of the hyperparameter value, we kept all settings the same, except for the $\eta$. For the ablation studies of the divergence class, we conduct a hyperparameter search among $1/\eta\in\{0.25,1,4,16\}$. For all baseline methods, we adopt the default hyperparameters identified as optimal through the hyperparameter searches described in the original papers.

\paragraph{Code reasoning.}
The experiments on math reasoning tasks are all conducted on eight NVIDIA A100 80GB. We train our models on the dataset of DeepCoder \citep{deepcoder}, following the experimental protocol of \citep{flowrl}. We use DeepSeek-R1-Distill-Qwen-7B as our LLM backbone, and Jensen-Shannon divergence to implement LAD loss in the main experiments. The batch size is set to 64, and the learning rate is set to 1e-6. For all baseline methods, we adopt the default hyperparameters identified as optimal through the hyperparameter searches described in the original papers. We evaluate the code reasoning performance on LiveCodeBench \citep{livecodebench}, CodeForces \citep{codeforces}, and HumanEval+ \citep{humaneval}. 

\paragraph{Evaluation metrics.}
Pass@k is a common metric in reasoning tasks. It measures the probability that at least one correct solution appears among the top k generated outputs. Higher Pass@k means the model is more likely to produce a correct answer within its top candidates.
Avg@k represents the average accuracy over the top k generated outputs. Unlike Pass@k, which is binary (pass/fail), Avg@k captures overall quality across multiple candidates.
Lexical diversity measures the variety of vocabulary used in a text. Higher lexical diversity indicates less repetition and a broader range of words, often used to assess richness or creativity in language generation.
Logical Diversity evaluates the variety of reasoning patterns, argument structures, or logical paths in generated outputs. Higher logical diversity suggests the model can approach a problem from multiple valid reasoning perspectives rather than repeating the same logic.

\subsection{Experimental Results}
\label{sec:apd:exp-results}

\paragraph{Supplemented results of the main experiments on math reasoning.}
As shown in Table~\ref{tab:main-pass}, for the evaluation metric Pass@32. Across six benchmarks, LAD achieves the best average performance (65.76), outperforming the second-best method (63.13).

\paragraph{Supplemented results of the ablation studies on divergence classes.}
As shown in Table~\ref{tab:append-eta-ablation}, the performances of LAD variants with stricter divergence classes are generally better than ones with non-stricter divergence classes.

\paragraph{Supplemented results of the ablation studies on the hyperparameter $\eta$.}
As shown in Table~\ref{tab:append-eta-ablation}, LAD can perform well in a large range of $\eta$ values.  Notably, across nearly all settings, LAD consistently outperforms GRPO, indicating that its effectiveness is not critically dependent on precise tuning of $\eta$.

\renewcommand{\arraystretch}{0.9}
\setlength{\tabcolsep}{9pt}
\begin{table*}[t!]
    \centering
      \caption{Performance of different losses on math reasoning benchmarks. The best two results under a specific metric are highlighted in \textbf{bold} and \underline{underlined}.}
    \label{tab:apd-div}
    
    \begin{tabular}{cc|ccccccc}
    \toprule
        Methods & $k$ & MATH & \makecell[c]{Olympiad\\Bench}& \makecell[c]{Minerva}& \makecell[c]{AIME\\2024}&\makecell[c]{AIME\\2025}& AMC& Average  \\
        \midrule
        &&\multicolumn{7}{c}{Avg@$k$}\\
        \multirow{3}{*}{LAD-KL} & @1& 76.80& 37.33& 38.97& 20.00& 5.83& 51.13& 38.01  \\
        & @16& 76.35& 36.93& 39.18& 20.03& 4.64& 51.74& 38.15 \\
        & @32& 76.20& 36.83& 38.73&19.91& 4.68& 51.63& 37.99\\
        \midrule
        \multirow{3}{*}{LAD-JF} & @1&72.60&32.74&37.13& 15.52 & 7.50& 53.69 & 36.20 \\
        & @16&72.65&32.54& 34.92& 15.96& 7.23& 53.47 & 36.13\\
        & @32&72.83&32.67& 34.60& 15.79 & 7.24& 53.48 & 36.10 \\
        \midrule
        \multirow{3}{*}{LAD-rKL} & @1&73.20&37.62& 37.50& 11.67& 8.33& 46.15& 35.41 \\
        & @16& 73.86& 36.63& 36.28&12.21& 8.03&45.93&35.49\\
        & @32 & 73.82& 36.76 & 36.12& 12.25& 8.10& 46.08& 35.52\\
        \midrule

        \multirow{3}{*}{LAD-TV} & @1& \underline{76.80}&\underline{40.59}&35.66&15.31&6.25&55.61&38.37\\
        &@16&77.00&\underline{40.92}&\textbf{38.14}&14.98&7.58&55.27&38.98 \\
        &@32&76.83&40.66&\textbf{37.74}&14.96&\underline{7.61}&55.19& 38.83 \\
        \midrule
        
        \multirow{3}{*}{LAD-HD} & @1& \textbf{77.60}& 39.70 & 35.29 & \underline{15.63}&\underline{7.92}& \textbf{57.34}& \underline{38.91}\\
        &@16& 76.88 & 39.06 & \underline{36.83} & \underline{16.27}& \underline{8.08}&\textbf{57.40}& \underline{39.08}\\
        &@32& 76.83 & 38.98 & \underline{36.74} & \underline{16.18}& 7.29& \textbf{57.40}& \underline{38.90}\\
        \midrule
        
        \multirow{3}{*}{LAD-JS} & @1&\underline{76.80}&40.30&\underline{36.40}&\textbf{19.38}&\textbf{8.65}&\underline{57.00}&\textbf{39.76}  \\
        &@16&\textbf{77.60}&\textbf{41.29}&36.81&\textbf{19.59}&\textbf{8.71}&\underline{56.37}& \textbf{40.06}\\
        &@32&\textbf{77.66}&\textbf{41.25}&36.68&\textbf{19.60}&\textbf{8.84}&\underline{56.45}& \textbf{40.08}\\
        \midrule
        &&\multicolumn{7}{c}{Pass@$k$}\\
        \multirow{3}{*}{LAD-KL} & @1&86.40& 51.11& 56.25& 34.48& 17.92&66.19&52.39 \\
        & @16& 87.80&55.26& 60.29& 39.69& 22.50& 71.16 & 56.12\\
        & @32& 88.60& 57.19& 61.40& 43.85& 26.46& 76.02 & 58.92 \\
        \midrule
        \multirow{3}{*}{LAD-JF} & @8& 88.2& 51.56& 54.04& 32.50 & 21.25& 77.18 & 54.12\\
        & @16& 90.4& 56.74& 58.09& 38.85& 26.98& 82.38 & 58.90\\
        & @32& 92.2& 60.59& 59.93& 46.04& 33.85& 86.71 & 63.22 \\
        \midrule
        \multirow{3}{*}{LAD-rKL} & @8&89.60& 58.67& 53.68& 29.06& 25.41& 72.29& 54.78 \\
        & @16& 91.20& 64.30& 59.93& 34.90& 30.83& 79.97& 60.19\\
        & @32& 93.60& 68.89& 62.87& 43.75& 37.50&86.07&65.45   \\
        \midrule

        \multirow{3}{*}{LAD-TV} & @8&87.60&55.70&\textbf{55.51}&26.35&\textbf{26.04} &73.53& 54.12\\
        &@16 &90.20&58.96&\textbf{60.29}&29.79&\underline{31.15}&77.67&58.01\\
        &@32&91.60&62.07&\textbf{62.50}&34.06&35.00&81.36&61.10\\
        \midrule
        
        \multirow{3}{*}{LAD-HD} & @8&88.20&56.59&\underline{54.41}&\underline{31.04}&\underline{25.42}&\underline{75.00} &\underline{55.11} \\
        &@16& 91.00&60.89&\underline{59.93}& \underline{36.04}&\textbf{31.67}&\underline{79.71}& \underline{59.87}\\
        &@32& 93.00& 64.74&\textbf{62.50}& \underline{42.71}&\textbf{38.23}&\underline{84.22}& \underline{64.23}\\
        \midrule
        \multirow{3}{*}{LAD-JS} & @8&\underline{89.00}&\underline{59.56}&54.04&\textbf{35.00}&\underline{25.42}&\textbf{78.58}&\textbf{56.93}\\
        &@16&\underline{92.40}&\underline{64.15 }&56.62&\textbf{40.63}&30.73&\textbf{84.30}& \textbf{61.47}\\
        &@32&\underline{93.80}&\underline{67.70} & 59.93&\textbf{47.19}&\underline{37.08}&\textbf{88.86}& \textbf{65.76}\\
        \bottomrule
    \end{tabular}
\end{table*}

\begin{figure*}[t]
\centering

\begin{mybox}[GPT-4 prompt for diversity judgement]
\label{box:gptprompt}
<System Prompt>\\
You are evaluating the DIVERSITY of solution approaches for a mathematics competition problem. Focus on detecting even SUBTLE differences in methodology that indicate different problem-solving strategies.

<User Prompt>\\
PROBLEM:\\
\{problem\}\\
16 SOLUTION ATTEMPTS:\\
\{formatted responses\}\\
EVALUATION CRITERIA - Rate diversity from 1 to 5:\\

Score 1 - Minimal Diversity:\\
• 14+ responses use essentially identical approaches\\
• Same mathematical setup, same variable choices, same\\ solution path\\
• Only trivial differences (arithmetic, notation, wording)\\
• Indicates very low exploration/diversity in the generation process\\

Score 2 - Low Diversity:\\
• 11-13 responses use the same main approach\\
• 1-2 alternative approaches appear but are rare\\
• Minor variations within the dominant method (different substitutions, orderings)\\
• Some exploration but heavily biased toward one strategy\\

Score 3 - Moderate Diversity:\\
• 7-10 responses use the most common approach\\
• 2-3 distinct alternative approaches present\\
• Noticeable variation in problem setup or mathematical techniques\\
• Balanced mix showing reasonable exploration\\

Score 4 - High Diversity:\\
• 4-6 responses use the most common approach\\
• 3-4 distinct solution strategies well-represented\\
• Multiple mathematical techniques and problem framings\\
• Strong evidence of diverse exploration strategies\\

Score 5 - Maximum Diversity:\\
• No single approach dominates ($\leq$ 3 responses use same method)\\
• 4+ distinctly different solution strategies\\
• Wide variety of mathematical techniques and creative approaches\\
• Excellent exploration and generation diversity\\

IMPORTANT: Focusing on the DIVERSITY of the attempted approaches. Return ONLY a number from 1 to 5.
\end{mybox}
\end{figure*}

\end{document}